\documentclass[11pt]{article}
\usepackage{times}
\usepackage{helvet}
\usepackage{courier}
\usepackage[hyphens]{url}
\usepackage{graphicx}
\usepackage{everyshi} %headers?
\usepackage{fancyhdr} %headers?

\usepackage[margin=1in]{geometry}
\usepackage{tikz}
\usetikzlibrary{arrows,fit}
\usetikzlibrary{positioning}
\usetikzlibrary{calc}
\usetikzlibrary{shapes,decorations}
	\tikzset{
		inner sep=0pt, outer sep=0pt, minimum size=0pt, thick,
		level/.style={sibling distance = (\columnwidth/16)*2^(4-#1)},
		winner/.style={minimum size=1.5em, circle, draw, fill=white, font={\footnotesize}},
		split/.style={minimum size=1.5em, inner sep=1pt, circle split, draw, fill=white, font={\tiny}},
		leaf/.style={inner sep=.15em, font={\footnotesize}},
		ball/.style={minimum size=.4em,circle,fill=black},
		beats/.style={thick,->,>=stealth',draw},
		tline/.style={thick,draw}
	}

\definecolor{light-gray}{gray}{0.9}
\usepackage{amsmath}
\usepackage{amssymb}
\usepackage{amsthm}
\usepackage{hyperref}
\hypersetup{
    colorlinks = true,
    linkcolor = black,
    urlcolor = blue,
    citecolor = green
}
\usepackage{nccmath} %fleqn
\usepackage{nicefrac}
\usepackage[font=small]{subcaption}
\usepackage[font=small]{caption}
\usepackage{multirow}
\newtheorem{theorem}{Theorem}%[section]
\newtheorem{lemma}[theorem]{Lemma}
\newtheorem{corollary}[theorem]{Corollary}

\newtheorem{definition}[theorem]{Definition}
\newtheorem{decision_rule}{Weighted Majority Rule}
\newtheorem{condition}[theorem]{Condition}

\newtheorem{claim}[theorem]{Claim}

\usepackage{color}
\usepackage[textsize=scriptsize]{todonotes}

\usepackage{xspace}
\newcommand{\voters}{\ensuremath{\mathcal{V}}\xspace} %set of voters
\newcommand{\cands}{\ensuremath{\mathcal{C}}\xspace} %set of candidates

\newcommand{\thresh}{\ensuremath{\tau}\xspace} %threshold
\newcommand{\threshl}{\ensuremath{\thresh_l}\xspace} %indexed threshold
\newcommand{\threshup}{\ensuremath{\thresh_{l+1}}\xspace} %next biggest indexed threshold

\newcommand{\ds}{\ensuremath{\delta}} %distortion
\newcommand{\actual}{\ensuremath{\ds_I}} %distortion

\newcommand{\icd}{\ensuremath{\Delta}} %ideal candidate distortion

\newcommand{\ai}{\ensuremath{\alpha_i}\xspace} %preference strength of voters who prefer winning candidate
\newcommand{\sumL}{\ensuremath{\sum\limits} \xspace}

\renewcommand{\P}{\ensuremath{P}\xspace} %winning candidate
\newcommand{\Q}{\ensuremath{Q}\xspace} %losing candidate

\DeclareMathOperator*{\argmin}{arg\,min} %to be able to use \argmax

\title{Awareness of Voter Passion Greatly Improves\\ the Distortion of Metric Social Choice}
\author{Ben Abramowitz \and Elliot Anshelevich \and Wennan Zhu}

\date{{\small Rensselaer Polytechnic Institute, Troy, NY\\ \today}}

\begin{document}
\maketitle

\abstract{We develop new voting mechanisms for the case when voters and candidates are located in an arbitrary unknown metric space, and the goal is to choose a candidate minimizing social cost: the total distance from the voters to this candidate. Previous work has often assumed that only ordinal preferences of the voters are known (instead of their true costs), and focused on minimizing distortion: the quality of the chosen candidate as compared with the best possible candidate. In this paper, we instead assume that a (very small) amount of information is known about the voter preference {\em strengths}, not just about their ordinal preferences. We provide mechanisms with much better distortion when this extra information is known as compared to mechanisms which use only ordinal information. We quantify tradeoffs between the amount of information known about preference strengths and the achievable distortion. We further provide advice about which type of information about preference strengths seems to be the most useful. Finally, we conclude by quantifying the {\em ideal candidate distortion}, which compares the quality of the chosen outcome with the best possible candidate that could ever exist, instead of only the best candidate that is actually in the running.
}

%\begin{center}
%``...it would be helpful for the development of democratic theory if we could assume that some means exist for comparing intensities of preferences. But do such means exist?'' - Robert Dahl (\cite{Dahl1956} pg. 99)
%\end{center}

\section{Introduction}
One often hears about `where candidates stand' on issues, calling to mind a spatial model of preferences in social choice \cite{arrow1990advances,merrill1999unified,hinich1984spatial,peter1990decade,schofield2007spatial}. In proximity-based spatial models, voters' preferences over candidates are derived from their distances to each of the candidates in some issue space. In particular, we consider voters and candidates which lie in an arbitrary unknown metric space. Our work follows a recent line of research in social choice which considers this setting  \cite{Anshelevich2018,anshelevich2017randomized,
%anshelevich2016blind,anshelevich2016truthful,Anshelevich2017,
anshelevich2018ordinal,borodin2019primarily,cheng2017people,cheng2018distortion,fain2019random,feldman2016voting,ghodsi2018distortion,goel2017metric,gross2017vote,pierczynski2019approval,skowron2017social}. The distance between each voter and the winning candidate is interpreted as the cost to that voter. Naturally, one of the main goals is to select the candidate which minimizes the total Social Cost, i.e., the sum of costs to the voters.

%A voting rule then aggregates these preferences to select a candidate who is close in some sense to the voters overall; a candidate who stands where they stand. %The quality of a candidate can therefore be expressed as a function over the distances between them and each of the voters.
%In the model we consider, voters and candidates are assumed to lie in an unknown metric space. The distance between each voter and the winning candidate is interpreted as the cost to that voter. Naturally, we want to select the candidate who minimizes the total Social Cost, or the sum of costs to the voters.

The crucial observation in the work cited above is that the actual costs of the voters for the selection of each candidate (i.e., the distances in the metric space) are often unknown or difficult to obtain \cite{BCHL+15a}. Instead, it is more reasonable to assume that voters only report {\em ordinal preferences}: orderings over the candidates which are induced by, and consistent with, latent individual costs. Because of this, past research has often focused on optimizing {\em distortion}: the worst-case ratio between the winning candidate selected by a voting rule aware of only ordinal preferences, and the best available candidate which minimizes the overall social cost. Many insights were obtained for this setting, including that there are deterministic voting rules which obtain a distortion of at most a small constant (5 in \cite{Anshelevich2018}, and more recently 4.236 in \cite{mungala2019improved}), and that no deterministic rule can obtain a distortion of better than 3 given access to only ordinal information.\footnote{We focus on deterministic mechanisms in this paper; see Related Work for discussion of why.}

The fundamental assumption and motivation in the above work is that the \emph{strength} or intensity of voter preferences is not possible to obtain, and thus we must do the best we can with only ordinal preferences. And indeed, knowing the exact strength of voter preferences is usually impossible. In many settings, however, {\em some} cardinal information about the ardor of voter preferences is readily available or obtainable, and is often used to affect outcomes and make better collective decisions. For example, a decision in a meeting may be decided in favor of a minority position if those in the minority are significantly more adamant or passionate about the issue than the apathetic majority, as revealed during discussion or debate. In political campaigns, the amounts of monetary donations, activists attending rallies, and other measures of ``grass-root support" can cause a candidate to become a de-facto front-runner even before an official election or primary is ever held. Because of this, in this paper we ask the question: ``How much can the quality of selected candidates be improved if we know some {\em small} amount of information about the {\em strength} of voter preferences?"

%By contrast, when using only a simple majority vote the raised hands of the most and least vehement voters are indistinguishable and the volumes of the `ayes' and `nays' go ignored. To make up for this shortcoming, some groups have turned to non-standard techniques like ``humming'' and ``rough consensus'' to make better collective decisions \cite{Resnick2014Humming}.

%Intuitively, preference rankings and approvals do not capture the \emph{strength} or intensity of voters' preferences. But in many settings some cardinal information about the ardor of voters' preferences is readily available or accessible through elicitation. For example, a decision in a meeting may be decided in favor of a minority position if those in the minority are significantly more adamant or passionate about the issue than the apathetic majority, as revealed during discussion or debate. By contrast, when using only a simple majority vote the raised hands of the most and least vehement voters are indistinguishable and the volumes of the `ayes' and `nays' go ignored. To make up for this shortcoming, some groups have turned to non-standard techniques like ``humming'' and ``rough consensus'' to make better collective decisions \cite{Resnick2014Humming}.

There are many different approaches modeling, measuring, eliciting, and aggregating the strength or intensity of voter preferences \cite{campbell1973social,farquhar1989preference}. Such measures can be done through survey techniques, measuring the total amount of monetary contributions, amounts of excitement and time people spend volunteering or advocating for particular issues, etc (see Related Work). All such measures are by their very nature imprecise. And yet while it is unreasonable to assume that exact strength of preference is known for every voter, it is certainly possible to obtain insights such as ``there are many more voters who are passionate about candidate A as compared to candidate B", or quantify the approximate amount of extreme preference strengths as opposed to the voters who are mostly indifferent. As we show in this paper, even such a small amount of information about aggregate preference strengths or the amount of passionate voters can greatly improve distortion, and allow mechanisms which provably result in outcomes which are close to optimal. In fact, knowing only a single additional bit of information for each voter (i.e., do they prefer A to B strongly, or not strongly?) is enough to greatly improve distortion.

\subsubsection*{Model and Notation}
As in previous work on metric distortion, we have a set of voters $\voters = \{1, 2, \ldots, n\}$ and a set of candidates (or alternatives) \cands. These voters and candidates correspond to points in an arbitrary (unknown) metric space $d$. %, i.e., for any three points $x, y, z$ triangle inequality holds: $d(x,y) \leq d(x,z) + d(y,z)$.\\
%Preferences and preference strengths induced by metric
The voter preferences over the candidates are induced by the underlying metric, i.e.,  %, and that voters are truthful (e.g. non-strategic).
%That is,
voters prefer candidates who are closer to them. Voter $i$ prefers candidate \P over candidate \Q (i.e., $\P \succ_i \Q$) only if $d(i,\P) \leq d(i,\Q)$. Moreover, we assume that the strengths of voter preferences are induced by these latent distances. If $i$ prefers \P over \Q, then the strength of this preference is $\ai^{\P\Q} =  \frac{d(i,\Q)}{d(i,\P)}$. %When it is clear we are referring to two candidates \P and \Q, we will drop the superscript. %Define social cost
The cost to voter $i$ if candidate \P is elected is $d(i,P)$, and the goal is to select the candidate minimizing the Social Cost:  $SC(P) = \sumL_{i \in \voters} d(i,P)$.

%Indexing (i vs. j)(alpha vs. beta)
%For clarity, when considering two candidates we index those who prefer \P using $i$ and denote their preference strength as \ai while we index those who prefer \Q using $j$ and denote their preference strength as \bj. Thus, $\forall j \in \voters$ if $d(j,\Q) \leq d(j,\P)$ then $\bj = \frac{d(j,\P)}{d(j,\Q)}$.\\

%Notation for thresholds and partitioning voters
%In the most general version of our model, we are given a set of preference strength thresholds $\{1 \leq \thresh_1 < \thresh_2 < \ldots < \thresh_m\}$, voters report the largest threshold which their preference strength exceeds for each pair of candidates. We let $A_l = \{i \in \voters : d(i,\P) \leq d(i,\Q)$ and $\thresh_{l} \leq \ai < \thresh_{l+1}\}$ and $B_l = \{j \in \voters : d(j,\Q) \leq d(j,\P)$ and $\thresh_{l} \leq \bj < \thresh_{l+1} \}$. For brevity, we allow $\thresh_{m+1} = \infty$. When $\thresh_1 = 1$ we know the preferred candidate of every voter. When $\thresh_1 > 1$ we let $C$ denote the set of voters with preference strength strictly less than $\thresh_1$ whose preferred candidate is unknown. When $m \rightarrow \infty$, we know the exact preference strength of every voter for every pair of candidates.\\

In previous work on metric distortion only the ordinal preferences were known, i.e., whether $(\P \succ_i \Q)$ or $(\Q \succ_i \P)$. In this paper, however, we assume that we are also given some information about the preference strengths $\ai^{\P\Q} =  \frac{d(i,\Q)}{d(i,\P)}$ as well. Note that knowing these values still does not tell us how $d(i,P)$ compares with $d(j,P)$ for $i\neq j$, only how strongly each voter feels when comparing different candidates. In fact, while even knowing the {\em exact} preference strengths of all the voters is not enough to be able to select the optimum candidate (as we show in this paper), knowing just one bit of information about $\ai^{\P\Q}$ (such as whether $\ai^{\P\Q}\geq\thresh$ for a threshold $\thresh$) is enough to create mechanisms with much better distortion.

%We would like to select the candidate with the minimum social cost. However, preference strength information is insufficient for any mechanism to guarantee selection of the best available candidate.\\

%Distortion of a candidate
For a given voting rule $\mathcal{R}$ and instance $I = \{\voters, \cands, d\}$, let $\P_I$ be the winning candidate selected by $\mathcal{R}$ and let $Z_I$ be the best available candidate (the one minimizing the Social Cost). Then, the \emph{distortion of winning candidate $P_I$} is defined as $$\actual = \frac{SC(P_I)}{SC(Z_I)}$$

The \emph{distortion of a voting rule $\mathcal{R}$} is defined its behavior on a worst-case instance: $$\ds = \max\limits_I \actual = \max\limits_I \frac{SC(P_I)}{SC(Z_I)}$$

%Our primary goal is find mechanisms which minimize \emph{distortion} (\ds), the worst-case approximation ratio between the social cost of the winning candidate and the best available candidate, over all possible instances.

\subsubsection*{Our Contributions}
What type of knowledge of the strengths of voter preferences is most useful and advantageous? What voting mechanisms should be used in order to minimize distortion if you have access to more information than only ordinal preferences? If you could gather data about voter preferences in different ways, what should you aim for in order to reduce distortion? These are some of the questions which we attempt to illuminate in this paper.

In this work, we study the possible distortion with different levels of voter preference strength information. A summary of our results is shown in Table \ref{table_results_distortion}. We begin with the setting in which we are given the voters' ordinal preferences, as well as a threshold $\thresh \ge 1$ of voter preference strength. In other words, for any two candidates $P$ and $Q$, we know the number of voters who prefer $P$ to $Q$, as well as how many of them prefer $P$ to $Q$ by at least a factor of $\tau$ (i.e., $d(i, P) < \frac{1}{\thresh}d(i, Q)$). Based on only this information about the voter preferences (and the fact that the voters and candidates are embedded in some arbitrary unknown metric space), we are able to provide new voting mechanisms with much better distortion than possible when only knowing ordinal preferences.
%and any voter $i$, we know whether $i$ prefer $P$ to $Q$ or prefer $Q$ to $P$. Without loss of generality, suppose $i$ prefer $P$ to $Q$, then we also know whether $i$'s preference strength for $P$ is greater than $\tau$ or smaller than $\tau$. Formally, for any voter $i$ that prefer $P$ to $Q$, $i$'s preference strength for $P$ is greater than $\tau$ in a metric space means that $d(i, P) < \frac{1}{\thresh} d(i, Q)$, and $i$'s preference strength for $P$ is no more than $\tau$ in a metric space means that $\frac{1}{\thresh} d(i, Q) \le d(i, P) \le d(i, Q)$.
For the case that there are only two candidates, we provide a mechanism which achieves provably best possible distortion of $\max\{ \frac{\thresh+2}{\thresh}, \frac{3\thresh - 1}{\thresh + 1}\}$, as shown in Figure \ref{fig:two_candidates_1_tau}. For the setting with more than two candidates, we get a distortion of $\min \{ \max\{\frac{3\thresh - 1}{\thresh + 1}, \frac{\thresh + 2}{\thresh} \} + 2, \max \{(\frac{3\thresh - 1}{\thresh + 1})^2, (\frac{\thresh + 2}{\thresh})^2 \} \}$ as shown in Figure \ref{fig:multi_candidates_1_tau}. Note that when $\thresh = 1$, we get a distortion of 5. A recently paper shows a deterministic algorithm that gives a distortion of 4.236. We believe our result can be improved using similar mechanisms to start the curve in Figure \ref{fig:multi_candidates_1_tau} from 4.236.

\begin{table}[htb]
\centering
\begin{tabular}{ | l | c | c | }
	\hline
	Distortion & Two Candidates & More than Two Candidates \\ \hline
	% Preferences and a threshold $\tau$ & $max\{ \frac{\thresh+2}{\thresh}, \frac{3\thresh - 1}{\thresh + 1}\}$ & $\min \{ \max\{\frac{3\thresh - 1}{\thresh + 1}, \frac{\thresh + 2}{\thresh} \} + 2,$ \\
  % &  & $\max \{(\frac{3\thresh - 1}{\thresh + 1})^2, (\frac{\thresh + 2}{\thresh})^2 \} \}$ \\\hline
	\multirow{2}{*}{Preferences and a threshold $\tau$} & \multirow{2}{*}{$\max\{ \frac{\thresh+2}{\thresh}, \frac{3\thresh - 1}{\thresh + 1}\}$} & \multicolumn{1}{|c|}{$\min \{ \max\{\frac{3\thresh - 1}{\thresh + 1}, \frac{\thresh + 2}{\thresh} \} + 2,$} \\
                          & & \multicolumn{1}{|c|}{$\max \{(\frac{3\thresh - 1}{\thresh + 1})^2, (\frac{\thresh + 2}{\thresh})^2 \} \}$} \\ \hline
	$m$ thresholds $\tau_1, \dots, \tau_m$ & $\max\limits_{1 \leq l \leq m} \{ \frac{\threshl \threshup + 2\threshup - 1}{\threshl \threshup + 1} \}$ & $\max\limits_{1 \leq l \leq m} \{ (\frac{\threshl \threshup + 2\threshup - 1}{\threshl \threshup + 1})^2 \}$ \\ \hline
	Exact preference strengths     & $\sqrt{2}$ & 2 \\ \hline
\end{tabular}
\caption{Distortion in different settings.}
\label{table_results_distortion}
\end{table}

From Figure \ref{fig:two_candidates_1_tau} and \ref{fig:multi_candidates_1_tau}, we can see that the distortion is minimized when $\tau = 1+\sqrt{2}$ in both settings. With only voter preferences being known, the best known deterministic distortion bounds are 3 for two candidates \cite{Anshelevich2018}, and 4.236 for multiple candidates \cite{mungala2019improved}. Interestingly, if we are also allowed to a choose a threshold $\thresh$, our results indicate that the optimal thing to do is to differentiate between candidates with lots of supporters who prefer them at least $1+\sqrt{2}$ times to other candidates, and candidates which have few such supporters. By obtaining this information, we can improve the quality of the chosen candidate from a 3-approximation to only a 1.83 approximation (for 2 candidates), and from a 4.236-approximation to a 3.35-approximation (for $\geq 3$ candidates). This is a huge improvement obtained with relatively little extra cost in information gathering.

\begin{figure}[htb]
\centering
\begin{minipage}{.42\linewidth}
  \includegraphics[width=\linewidth]{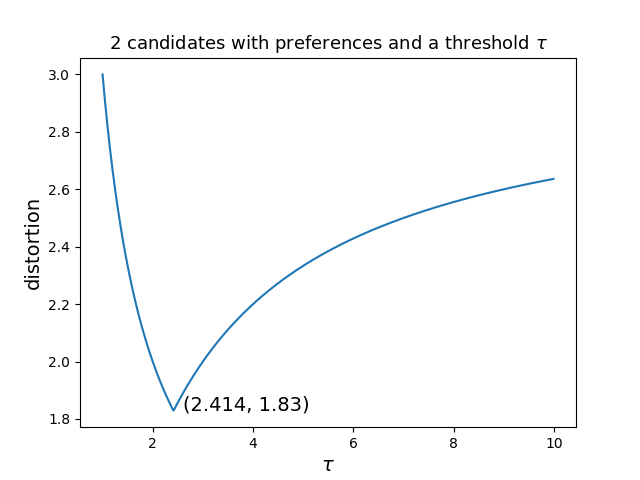}
  \captionof{figure}{Distortion for two candidates with preferences and a threshold $\thresh$.}
  \label{fig:two_candidates_1_tau}
\end{minipage}
\hspace{.04\linewidth}
\begin{minipage}{.42\linewidth}
  \includegraphics[width=\linewidth]{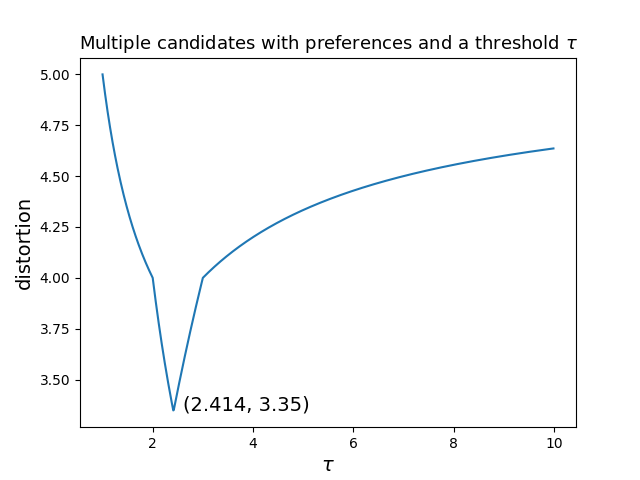}
  \captionof{figure}{Distortion for more than two candidates with preferences and a threshold $\thresh$.}
  \label{fig:multi_candidates_1_tau}
\end{minipage}
\end{figure}

In Section \ref{sec-1tau} we consider the case when we only know the preferences of voters who feel strongly about their choice (prefer $P$ to $Q$ by at least $\thresh$ times), but do not know the preferences of voters who are relatively indifferent. We show that knowing how many voters feel strongly about a candidate is actually {\em more} important than knowing the ordinal preferences of all voters when attempting to minimize distortion: for example if we have $\tau=2$ we can obtain a distortion of 2 as well, even if we don't know the preferences of all voters.

We then consider a more general case in Section \ref{sec_general}. Suppose we have $m$ different thresholds $\{1 \leq \thresh_1 < \thresh_2 < \ldots < \thresh_m\}$, and voters report the largest threshold which their preference strength exceeds for each pair of candidates. As $m$ gets larger, the information about preference strengths gets less coarse; for most settings it would be realistic to assume that $m$ is small, but we provide a result which is as general as possible. With this information, we give a mechanism achieving the provably best distortion of $\max\limits_{1 \leq l \leq m} \{ \frac{\threshl \threshup + 2\threshup - 1}{\threshl \threshup + 1} \}$ in the two candidates setting, and a distortion of $\max\limits_{1 \leq l \leq m} \{ (\frac{\threshl \threshup + 2\threshup - 1}{\threshl \threshup + 1})^2 \}$ in the multiple candidates setting. Note that knowing all the preference strengths {\em exactly} is still not enough to always be able to choose the optimum candidate: the preference strengths are relative (``I like A twice as much as B") as opposed to absolute. We never obtain information about how the costs of different voters compare to each other, the only thing we know is that the voters lie in a metric space. In fact, when we know the exact preference strengths of every voter, we obtain a distortion bound of $\sqrt{2}$ in the two candidates setting, and a distortion of $2$ in the multiple candidates setting. Moreover, we prove that even knowing the exact preference strengths, it is not possible to obtain distortion better than $\sqrt{2}$ in the worst case.

\paragraph{Ideal Candidate Distortion}
In addition to forming mechanisms with small distortion, we also have a secondary goal in this paper. Rather than only comparing the winning candidate to the best available candidate, we can also measure them against the ideal conceivable candidate $Z_I^*$ who may not be an available option to vote upon. $Z_I^*$ is the point in the metric space which minimizes social cost; it is the absolute best consensus of the voters, and it would be wonderful if that point corresponded to a candidate, but that may not be the case (i.e., $Z_I^*$ may not be in \cands). We introduce the notion of \emph{ideal candidate distortion} as follows, where $I=\{\voters, \cands, d\}$ is any instance and $P_I$ is the winner that our mechanism selects for instance $I$:

$$\icd = \frac{SC(P_I)}{SC(Z_I^*)}$$

As we show, while the ideal candidate distortion $\icd$  is unbounded in general, for many simple voting rules it can be bounded as a function of the distortion of the winning candidate ($\actual$). Intuitively, the distortion $\actual$  can only be high when the best available candidate (best in $\cands$) is close to being the ideal possible candidate (best in the entire metric space).

%Finally, we introduce the notion of \emph{ideal candidate distortion} $\icd$, and study the tradeoff between the distortion and the ideal candidate distortion $\icd$.
A summary of our results on this topic is shown in Table \ref{table_results_ideal}. %Suppose the winner is $P$, and the optimal candidate is $Q$. Denote the distortion of $P$ as $\actual = \frac{SC(P)}{SC(Q)}$. Note that we use $\ds$ to represent the upper bound of distortion, and $\actual$ to represent the actual distortion. Remember we assume the candidates and voters are in a metric space. There must exist a point in this metric space, that if we put a candidate right on top of it, then the total social cost if minimized. We denote this ideal candidate as $Z^*$, and define the ideal candidate distortion $\icd$ of a candidate $P$ as $\icd = \frac{SC(P)}{SC(Z^*)}$.
%In the two candidates setting with only voters's preferences, $\icd \le \frac{2 \actual}{\actual - 1}$, and the in the multiple candidates setting, $\icd \le \frac{4 \actual}{\actual - 1}$. In the two candidates setting with voters's preferences and a threshold $\thresh$, $\icd \le \frac{2 \actual}{\actual - 1}$, and the in the multiple candidates setting, $\icd \le \frac{4 \actual}{\actual - 1}$. With every voters' exact preferences, $\icd \le \frac{(\sqrt{2} +1) \actual}{\actual - 1}$ for two candidates, and $\icd \le \frac{2(\sqrt{2} +1) \actual}{\actual - 1}$ for multiple candidates.
These results imply that if we are only given ordinal preferences, as in most previous work, and use certain mechanisms like the Copeland voting mechanism, then {\em either} the selected candidate is much closer to the best candidate in the running than the worst-case distortion bound indicates (say within factor of $\delta_I=3$ instead of the worst-case of 5 for the Copeland mechanism), {\em or} the selected candidate is not far from the {\em ideal} candidate, i.e., the best candidate that could ever exist (say within factor of 6 if $\actual=3$). So in the case when distortion is high, we at least can comfort ourselves with the fact that the selected candidate is not too far away from the best possible candidate that could ever exist, not just from the best candidate in the running.

 \renewcommand{\arraystretch}{1.3}

 \renewcommand{\arraystretch}{1.5}
\begin{table}[htb]
\centering
\begin{tabular}{ | l | c | c | }
	\hline
	Ideal Candidates Distortion & Two Candidates & Multiple Candidates \\ \hline
	Only preferences & $\frac{2 \actual}{\actual - 1}$ & $\frac{4 \actual}{\actual - 1}$ \\ \hline
	Preferences and a threshold \thresh& $\frac{2 \actual}{\actual - 1}$ & $\frac{4 \actual}{\actual - 1}$ \\ \hline
	Exact preference strengths  & $\frac{(\sqrt{2} + 1) \actual}{\actual - 1}$ & $\frac{2(\sqrt{2} + 1) \actual}{\actual - 1}$ \\ \hline
\end{tabular}
\caption{Ideal candidate distortion (\icd) bounds}
\label{table_results_ideal}
\end{table}

\section{Related Work and Discussion}
%GENERAL STUFF
%Ordinal approximations for implicit utilitarian objectives have been studied in many settings, including social choice, matchings, secretary problems, general graph problems, and more. These settings commonly assume that we have access only to the ordinal preferences of the agents induced by their underlying costs (resp. utilities), and the goal is to approximate the minimum social cost (resp. maximum social welfare). In this paper we consider social choice problems in which voters' ordinal preferences over candidates are assumed to lie in an unknown metric space. Unlike previous work, our voters provide a little bit of extra information about the strength of their preferences.

%DETERMINISTIC MECHANISMS
The concept of distortion was introduced by \cite{Procaccia2006} as a measure of efficiency for ordinal social choice functions (see also \cite{Anshelevich2018,BCHL+15a} for discussion). Since then, two main approaches have emerged for analyzing the distortion of various voting mechanisms. One is assuming that the underlying unknown utilities or costs are normalized in some way, as in e.g.,  \cite{benade2017preference,benade2019low,bhaskar2018truthful,BCHL+15a,CNPS17a,caragiannis2011voting}. The second approach, which we take here, assumes all voters and candidates are points in a metric space \cite{Anshelevich2018,anshelevich2017randomized,
%anshelevich2016blind,anshelevich2016truthful,Anshelevich2017,
anshelevich2018ordinal,borodin2019primarily,cheng2017people,cheng2018distortion,fain2019random,feldman2016voting,ghodsi2018distortion,goel2017metric,gross2017vote,pierczynski2019approval,skowron2017social}. In particular, when the latent numerical costs that induce voter preferences over a set of candidates obey the triangle inequality, it is known that simple deterministic voting rules yield distortion which is always at most a small constant (5 for the well-known Copeland mechanism  \cite{Anshelevich2018}, and recently 4.236 for a more sophisticated, yet elegant, mechanism  \cite{mungala2019improved}).
%\cite{Anshelevich2018} showed that Copeland yields a tight bound of 5, and \cite{goel2017metric} showed that the Ranked Pairs and Schulze rules obtains the same bound. More recently, \cite{mungala2019improved} identified a more sophisticated rule that yields an approximation of 4.236, and we believe our results may be improved using similar techniques.
While \cite{Anshelevich2018} showed that no deterministic mechanism can always produce distortion better than 3, closing this gap remains an open question.

%RANDOMIZED MECHANISMS & TRUTHFULNESS (DECISIVENESS?)
{\bf Randomized vs Deterministic Mechanisms~~} In this paper we restrict our attention to deterministic social choice rules, instead of randomized ones as in e.g., \cite{anshelevich2017randomized, CNPS17a, feldman2016voting, gross2017vote}, for several reasons. First, consider looking at our mechanisms from a social choice perspective, i.e., as voting rules that need to be adopted by organizations and used in practice. People are far more resistant to adopting randomized voting protocols. %, even when their {\em expected} outcomes are good \elliot{\cite{.}}.
This is because an election with a non-trivial probability of producing a terrible outcome is usually considered undesirable, even if the {\em expected} outcomes are good. There are many exceptions to this, of course, but nevertheless deterministic mechanisms are easier to convince people to adopt. Second, consider looking at our mechanisms from the point of view of approximation algorithms, i.e., as algorithms which attempt to produce an approximately-optimal solution given a limited amount of information. For traditional randomized approximation algorithms with guarantees on the quality of the expected outcome it is possible to run the algorithm several times, take the best of the results, and be relatively sure that you have achieved an outcome close to the expectation. In this setting of limited information, however, we cannot know the ``true" cost of a candidate even after a randomized mechanism chooses it, and thus cannot take the best outcome after several runs. Therefore, unless stronger approximation guarantees are given than simply bounds on the expectation, it is quite likely that the outcome of a randomized algorithm in our setting would be far from the expected value. While randomized algorithms are certainly worthy of study even in our setting, and many interesting questions about them exist, we choose to focus only on deterministic algorithms in this paper.

Attempts to exploit preference strength information have led to various approaches for modeling, eliciting, measuring, and aggregating people's preference intensities in a variety of fields, including Likert scales, semantic differential scales, sliders, constant sum paired comparisons, graded pair comparisons, response times, willingness to pay, vote buying, and many others (see \cite{campbell1973social,farquhar1989preference,gerasimou2019preference} for summaries). %Naturally, this has led to a variety of weighted majority rules to capture preference intensities in various full-information settings \cite{Azrieli2014,Fleurbaey2008Weighted}.
%Our model of voter preference strength is a type of preference intensity function as recently defined by \cite{gerasimou2019preference}, satisfying the criteria of order-preservation, skew-symmetry\footnote{For mathematical convenience, we use $\frac{d(i,Q)}{d(i,P)}$ as the measure of preference strength rather than $\frac{d(i,Q)}{d(i,P)}-1$ which exhibits skew-symmetry, but these are equivalent and the choice has no impact on our results.}, and lateral consistency, and is amenable to weighted majority rules.
In our work we specifically consider only a small amount of coarse information about preference strengths, since obtaining detailed information is extremely difficult.
Intuitively, any rule used to aggregate preference strengths must ask under what circumstances an `apathetic majority' should win over a more passionate minority \cite{Willmoore1968Intensity}, and we provide a partial answer to this question when the objective is to minimize distortion.

%DECISIVENESS
Perhaps most related to our work is that of \cite{anshelevich2017randomized} which introduced the concept of \emph{decisiveness}. Using our notation, \cite{anshelevich2017randomized} proves bounds on distortion under the assumption that {\em every} voter has a preference strength at least $\alpha$ between their top and second-favorite candidates. We, on the other hand, do not require that voters have any specific preference strength between any of their alternatives, and provide general mechanisms and distortion bounds based on knowing a bit more about voters (arbitrary) preference strengths. In other words, while \cite{anshelevich2017randomized} limits the possible space of voter preferences and locations in the metric space, we instead allow those to be completely arbitrary, but assume that we are given slightly more information about them.

%Most similar to our work is that of \cite{anshelevich2017randomized} which introduced decisiveness in the implicit utilitarian setting to characterize how strongly voters feel about their top choice compared with their second favorite choice. Like our notion of preference strength, it is based on the ratio of latent distances. By contrast, we consider the effect on distortion when we know something about the pairwise strength of a voters' preference between every pair of candidates, not only their top two. Also similar to our work is that of \cite{pierczynski2019approval}, which considers distortion when voters submit approval ballots rather than ordinal preferences. In their model, voters approve of all candidates who are within a ball of a certain radius of that voter. However, this differs significantly from our model in which voters who are the same distance from one of the candidates candidate can have different ordinal preferences and preference strengths.

%INFO-DISTORTION TRADEOFF
% More recent work has shed light on the tradeoff between the amount of information provided by voters and distortion \cite{Mandal2019Efficient}. %(+\textcolor{red}{Communication vs. Distortion in Metric Voting} not available yet).
% This work focuses on how distortion increases when complete ordinal preferences are not even available. We consider the opposite case, when we have a little bit more information than just ordinal preferences, and show that even a little bit of preference strength information can greatly reduce distortion.

In our model, when voter preference strength is less than the smallest threshold ($\thresh_1 > 1$), they effectively abstain because their preferred candidate is unknown, and so any reasonable weighted majority rule must assign them a weight of 0. Therefore, our work also bears resemblance to literature on voter abstentions in spatial voting (see \cite{ghodsi2018distortion} and references therein). While there are major technical differences in our model and that of \cite{ghodsi2018distortion}, at a high level the model of \cite{ghodsi2018distortion} is similar to a special case of ours with only two candidates and a single threshold on preference strengths (and no knowledge of voter preferences otherwise), which we analyze in Section \ref{sec-1tau}.

%However, the abstention of a voter in previous literature is typically based on their distance to the closest candidate or the difference between distances to candidates, or due to strategic behavior.

%STRATEGIC AGENTS
Finally, in this paper we assume that the preference strengths given to our algorithms are truthful, i.e., that the voters do not lie. While it would certainly be interesting and important to consider the case where voters may not be truthful (as in e.g., \cite{bhaskar2018truthful,feldman2016voting}), for many settings with preference strengths it is actually more reasonable to expect voters to be truthful than for settings with only ordinal votes. This is because preference strengths are often signaled passively (e.g., average response times to surveys) or expressing this intensity comes at a cost (e.g., time commitments, activism, or monetary contributions and payments). Even in debates and committees where a member signals their strong preference for A over B, this member is putting their reputation on the line in doing so, and so may not want to do this unless their preference is actually that strong, in order to not look foolish or inconsistent in the future.

\section{Preliminaries and Lower Bounds}
In our model we have a set of voters $\voters = \{1, 2, \ldots, n\}$ and a set of candidates \cands. These voters and candidates correspond to points in an arbitrary metric space $d$, so for any three points $x, y, z$ the triangle inequality holds: $d(x,y) \leq d(x,z) + d(y,z)$.
%Preferences and preference strengths induced by metric
We assume that voters' preferences over the candidates are induced by the underlying metric, and that voters are truthful (i.e., non-strategic). That is, voters prefer candidates who are closer to them. Voter $i$ prefers candidate \P over candidate \Q $(\P \succ_i \Q)$ only if $d(i,\P) \leq d(i,\Q)$. Moreover, we assume that the strength of voters' preferences are induced by these latent distances. If $i$ prefers \P over \Q, then the strength of this preferences is $\ai^{\P\Q} =  \frac{d(i,\Q)}{d(i,\P)}$. When it is clear we are referring to 2 candidates \P and \Q, we will drop the superscript.

%Indexing (i vs. j)(alpha vs. beta)
%For clarity, when considering two candidates we index those who prefer \P using $i$ and denote their preference strength as \ai while we index those who prefer \Q using $j$ and denote their preference strength as \bj. Thus, $\forall j \in \voters$ if $d(j,\Q) \leq d(j,\P)$ then $\bj = \frac{d(j,\P)}{d(j,\Q)}$.\\

%Notation for thresholds and partitioning voters
Given a set of preference strength thresholds $\{1 \leq \thresh_1 < \thresh_2 < \ldots < \thresh_m\}$, voters report the largest threshold which their preference strength exceeds for each pair of candidates. We let $A_l = \{i \in \voters : d(i,\P) \leq d(i,\Q)$ and $\thresh_{l} \leq \ai^{PQ} < \thresh_{l+1}\}$ and $B_l = \{j \in \voters : d(j,\Q) \leq d(j,\P)$ and $\thresh_{l} \leq \alpha_j^{QP} < \thresh_{l+1} \}$. For convenience, we say $\thresh_{m+1} = \infty$ and $\thresh_0 = 1/\thresh_1$. When $\thresh_1 = 1$ we know the preferred candidate of every voter. When $\thresh_1 > 1$ we let $C$ denote the set of voters with preference strength strictly less than $\thresh_1$ whose preferred candidate is unknown. When $m \rightarrow \infty$, we know the exact preference strength of every voter for every pair of candidates.

%Define social cost
We consider cost to voter $i$ if candidate \P is elected as $d(i,P)$ and the Social Cost is the sum of the costs to all of the individual agents, $SC(P) = \sumL_{i \in \voters} d(i,P)$. We would like to select the candidate with the minimum social cost. However, preference strength information is insufficient for any mechanism to guarantee selection of the best available candidate. Therefore, our primary goal is study and design mechanisms which minimize \emph{distortion} (\ds), the worst-case approximation ratio between the social cost of the candidate we select and the best available candidate over all possible instances, as defined in the Introduction.

% We would like to select the candidate with the minimum social cost. However, preference strength information is insufficient for any mechanism to guarantee selection of the best available candidate. Therefore, our primary goal is find mechanisms which minimize \emph{distortion} (\ds), the worst-case approximation ratio between the social cost of the candidate we select and the best available candidate over all possible instances. For a given voting rule, and for instances $I = \{\voters, \cands\}$, let $\P_I$ be the winning candidate and let $Z_I$ be the best available candidate.

% $$\ds = \max\limits_I \frac{SC(P_I)}{SC(Z_I)}$$

% Furthermore, we have a secondary goal. Rather than only comparing the winning candidate to the best available candidate, we can also measure them against the ideal conceivable candidate $Z_I^*$. We introduce the notion of \emph{ideal candidate distortion} as follows:

% $$\icd = \max\limits_I \frac{SC(P_I)}{SC(Z_I^*)}$$

% As we will show, while $\icd$ is unbounded in general, for many simple voting rules with optimal distortion, the ideal candidate distortion can be bounded as a function of the approximation ratio.

\subsection{Lower Bounds on Distortion with Preference Strengths} \label{sec_lb}
Here, we provide lower bounds on the minimum distortion any deterministic mechanism can achieve given only preference strength information. First, note that even if all {\em exact} preference strengths were known to us, we still would not be able to choose the optimum candidate: knowing the relative strength of preference for every voter is not the same thing as knowing their exact distances to every candidate (i.e., we would only know $\alpha_i=\frac{d(i,P)}{d(i,Q)}$ and not $d(i,P)$ and $d(i,Q)$ themselves).
%Note that since all examples above are 1D, no improvements to distortion can be made by restricting our focus to only Euclidean or 1D metric preferences, as is often done in the literature \cite{.}. However, our upper bounds hold for arbitrary metric spaces.
%Moreover, all of our worst-cases occur when voters are placed in either 1 or 2 locations.\\

\begin{theorem}\label{thm.lower_bound_exact}
No deterministic mechanism with only preference strength information can achieve a worst-case distortion less than $\sqrt{2}$.
\end{theorem}

\begin{proof}
The example used is in 1D, where candidates \P and \Q are represented by points on a line. We normalize the distances so that \P is at location 0 and \Q is at location 1. Suppose half the voters prefer \P with strength $1+\sqrt{2}$, and the other half prefer \Q with strength $1+\sqrt{2}$. Since this is the only information known to the mechanism, the mechanism must tie-break in some arbitrary way (if tie-breaking is undesirable, we can have one extra voter prefer \P, which will result in distortion arbitrarily close to $\sqrt{2}$ instead of exactly $\sqrt{2}$). Thus without loss of generality, we let \P be the winner over \Q.

Suppose the true location of the voters is as follows. Half of the voters are located at $\frac{1}{2+\sqrt{2}}$ and the other half are located at $\frac{3+2\sqrt{2}}{2+\sqrt{2}}$. All voters have a preference strength of $1+\sqrt{2}$. If there are $N$ voters, the candidates have social costs $SC(P) = 2N$ and $SC(Q) = \sqrt{2}N$. Thus, if \P wins we have a lower bound on distortion of $\ds \geq \sqrt{2}$.
\end{proof}

Of course it is unrealistic to expect to know the exact preference strengths of all the voters. Below we give a general lower bound for the best distortion possible given knowledge of certain preference thresholds.

\begin{theorem} \label{thm_mT_lowerbound}
When given knowledge of $m$ fixed thresholds, no deterministic mechanism can always achieve a distortion less than $\max\limits_{0 \leq l \leq m} \{ \frac{\thresh_l \thresh_{l+1} + 2 \thresh_{l+1} - 1}{\thresh_l \thresh_{l+1} + 1} \}$
%$\max\{ \frac{\thresh_m + 2}{\thresh_m}, \thresh_1, \max\limits_{2 \leq l \leq m} \frac{\thresh_l \thresh_{l-1} + 2 \thresh_{l} - 1}{\thresh_l \thresh_{l-1} + 1} \}$.
\end{theorem}

\begin{proof}
The proof follows from the following 3 lemmas. The examples used for these lemmas are all in 1D, where candidates \P and \Q are represented by points on a line. We normalize the distances so that \P is at location 0 and \Q is at location 1 and use $\epsilon$ to denote an infinitesimal quantity. Without loss of generality, we let \P be the winner over \Q. Recall that we have defined $\thresh_0 = \frac{1}{\thresh_1}$ and $\thresh_{m+1} = \infty$ for convenience.

\begin{lemma} \label{lem_smallest_lowerbound}
If we have a set of thresholds of which the smallest is $\thresh_1 > 1$, no deterministic mechanism can always achieve a distortion less than $\thresh_1$.
\end{lemma}

\begin{proof}
Suppose all $N$ voters are located at position $\frac{\thresh_1}{\thresh_1+1} - \epsilon$. All voters therefore have preference strength less than $\thresh_1$, so $|C| = N$ and the preferred candidates of the voters are unknown. If \P wins over \Q due to tie-breaking, as $\epsilon \rightarrow 0$ this yields a lower bound on distortion of $\ds \geq \thresh_1$.
%Add diagram of example
\end{proof}

\begin{lemma} \label{lem_largest_lowerbound}
If we have a set of thresholds of which the largest is $\thresh_m > 1$, no deterministic mechanism can always achieve a distortion less than $\frac{\thresh_m + 2}{\thresh_m}$.
\end{lemma}

\begin{proof}
Suppose half of the voters are located on top of $Q$ at position 1 and the other half of voters are located at $\frac{1}{\thresh_m+1} - \epsilon$, so $|A_m| = |B_m| = \nicefrac{N}{2}$. If \P wins over \Q due to tie-breaking, as $\epsilon \rightarrow 0$ this yields a lower bound on distortion of $\ds \geq \frac{\thresh_m + 2}{\thresh_m}$.
%Add diagram of example
\end{proof}

\begin{lemma} \label{lem_2T_lowerbound}
If we have a set of thresholds, of which two consecutive thresholds are $\thresh_{l}$ and $\thresh_{l+1}$ where $\thresh_{l} < \thresh_{l+1}$, no deterministic mechanism can achieve a distortion less than $\frac{\thresh_{l} \thresh_{l+1} + 2 \thresh_{l+1} - 1}{\thresh_{l} \thresh_{l+1} + 1}$.
\end{lemma}

\begin{proof}
Suppose half of the voters are located at position $\frac{1}{\thresh_{l} + 1} - \epsilon$ and the other half are located at position $1+\frac{1}{\thresh_{l+1} - 1} + \epsilon$. Once again, the mechanism must choose randomly between the candidates because $|A_{l}| = |B_{l}| = \nicefrac{N}{2}$.
%The candidates have social costs $SC(P) = \frac{1}{\thresh_l + 1} + 1 + \frac{1}{\thresh_{l+1} - 1}$ and $SC(Q) = 1 - \frac{1}{\thresh_l + 1} + \frac{1}{\thresh_{l+1} - 1}$.
Therefore, if \P wins over \Q due to tie-breaking, as $\epsilon \rightarrow 0$, this yields a lower bound on distortion of $\ds \geq \frac{\thresh_{l} \thresh_{l+1} + 2 \thresh_{l+1} - 1}{\thresh_{l} \thresh_{l+1} + 1}$.
\end{proof}

The combination of the three preceding lemmas guarantees the lower bound of Theorem \ref{thm_mT_lowerbound}.
\end{proof}

\section{Adding the knowledge of a single threshold $\thresh$ to ordinal preferences}
\label{sec-1-tau}
%\section{Preferences and another threshold $\thresh$}

\subsection{Distortion with Two Candidates}
In this section we begin by analyzing the case with only two possible candidates. In the section that follows, we use these results to form mechanisms with small distortion for multiple candidates. Suppose there are two candidates $P$ and $Q$. We are given the users' ordinal preferences, and a strength threshold $\thresh$, i.e., for every voter we only know two bits of information: whether they prefer $P$ or $Q$, and whether their preference is strong ($>\thresh$) or weak ($\leq\thresh$). Note that our results still hold if we only have this knowledge in aggregate, i.e., if for both $P$ and $Q$ we know approximately how many people prefer $P$ to $Q$ strongly versus weakly, and vice versa.

Notice that preference strengths tell us little about the true underlying distances for voters with weak preference strengths, because the preference strength of a voter almost directly between \P and \Q who is very close to both can have the same preference strength as a voter who is very distant from both candidates. However, if a voter's preference strength is large, we know they must be fairly close to one of the candidates - and it is these passionate voters who contribute most to distortion.

\begin{decision_rule}
\label{weighted_majority_rule_1_tau}
 Given voters' preferences and a threshold $\thresh$ for two candidates, if $\thresh \ge \sqrt{2} + 1$, assign weight $\frac{\thresh+1}{\thresh-1}$ to all the voters with preference strengths $> \thresh$ and weight $1$ to all the voters with preference strengths $\le \thresh$. If $\thresh < \sqrt{2} + 1$, assign weight $\thresh$ to all the voters with preference strengths $> \thresh$ and weight $1$ to all the voters with preference strengths $\le \thresh$. Choose the candidate by a weighted majority vote.
\end{decision_rule}

The following theorem shows that the above voting rule produces much better distortion than anything possible from knowing only the ordinal preferences. Moreover, due to the lower bounds in the previous section, this is the best distortion possible (apply Theorem \ref{thm_mT_lowerbound} with $\tau_1=1$ and $\tau_2=\tau$).

\begin{theorem}
	\label{thm-1-thresh}
With 2 candidates in a metric space, if we know voters' preferences and a strength threshold $\thresh$, Weighted Majority Rule \ref{weighted_majority_rule_1_tau} has a distortion of at most $\ds = \max\{ \frac{\thresh+2}{\thresh}, \frac{3\thresh - 1}{\thresh + 1}\}$.
\end{theorem}

\begin{proof}
 Denote the set of voters prefer $P$ with preference strengths $>\thresh$ as $A_2$, and with preference strengths $\le \thresh$ as $A_1$. Also denote the set of voters prefer $Q$ with preference strengths $>\thresh$ as $B_2$, and with preference strengths $\le \thresh$ as $B_1$. Without loss of generality, suppose we choose $P$ as the winner by our weighted majority rule. It means that if $\thresh \ge \sqrt{2} + 1$, $\frac{\thresh+1}{\thresh-1} |A_2| + |A_1| \ge |B_1| + \frac{\thresh+1}{\thresh-1} |B_2|$, and for $\thresh < \sqrt{2} + 1$,   $\thresh |A_2| + |A_1| \ge |B_1| + \thresh |B_2|$.\\

\noindent Proof Sketch and Main Idea: For all voters, consider their individual ratio of $\frac{d(i,P)}{d(i,Q)}$, regardless of which candidate they prefer. For voters who prefer \P this is their preference strength, and for voters who prefer \Q this is the reciprocal of their preference strength. If for all voters this was less than \ds, then clearly we have a distortion of at most $\ds$ by just summing them up. However, for some voters this ratio is higher and for others it is lower. If we think of charging $SC(P)$ to $SC(Q)$, we should charge the voters for whom this ratio is lower to the voters for whom this ratio is higher. Clearly, for any voters who prefer \P this ratio is less than 1 and so it is less than \ds. For voters who prefer \Q, some voters with weak preferences will allow us to save charge while others with stronger preferences will use up the extra charge. However, charging the voters to other voters seems quite difficult in this setting. The main new technique in our proof is to use $d(P,Q)$ as a sort of numeraire or store of value. We first perform the charging for all voters for whom this ratio is small, and we use $d(P,Q)$ to quantify how much extra charge is saved. We then show that this quantity of charge stored in terms of $d(P,Q)$ is sufficient to expend the charge from the remaining voters, yielding a distortion at most \ds.

 We first show some lemmas to bound $d(i, P)$ by $d(i, Q)$ and $d(P, Q)$ for every voter $i$.

 \begin{lemma}
	 \label{lemma-1-thresh-A2}
	  $\forall i \in A_2$, for any $\delta \ge 1$, 	$d(i, P) \le \delta d(i, Q) - \frac{\delta \thresh - 1}{\thresh + 1} d(P, Q)$.
 \end{lemma}
 \begin{proof}

  $\forall i \in A_2$, $d(i, P) \le \frac{1}{\thresh} d(i, Q)$. By the triangle inequality,

	\begin{align*}
	d(P,Q) \le d(i, P) + d(i, Q) \le \frac{1}{\thresh}d(i, Q) + d(i, Q) = \frac{1 + \thresh}{\thresh}d(i, Q)
  \end{align*}

  Thus $d(i,Q) \ge \frac{\thresh}{ \thresh + 1} d(P, Q)$. $\forall i \in A_2$,
  \begin{align*}
		d(i, P) &\le \frac{1}{\thresh} d(i, Q) \\
						&= \delta d(i, Q) - (\delta - \frac{1}{\thresh}) d(i, Q) \\
					  &\le \delta d(i, Q) - \frac{\delta \thresh - 1}{\thresh + 1} d(P, Q)
	\end{align*}
\end{proof}

\begin{lemma}
	\label{lemma-1-thresh-A1}
	 $\forall i \in A_1$, for any $\delta \ge 1$,	$d(i, P) \le \delta d(i, Q) - \frac{\delta-1}{2} d(P, Q)$.
\end{lemma}

\begin{proof}

	$\forall i \in A_1$, by the triangle inequality,

	\begin{align*}
	d(P,Q) \le d(i, P) + d(i, Q) \le d(i, Q) + d(i, Q) = 2 d(i, Q)
  \end{align*}

  Thus $d(i, Q) \ge \frac{1}{2} d(P, Q)$. $\forall i \in A_1$,
	\begin{align*}
		d(i, P) &\le d(i, Q)\\
						&= \delta d(i, Q) - (\delta - 1) d(i, Q) \\
						&\le \delta d(i, Q) - \frac{\delta-1}{2} d(P, Q)
	\end{align*}
\end{proof}

\begin{lemma}
	\label{lemma-1-thresh-B1}
	 $\forall j \in B_1$, for any $1 \le \delta \le \thresh$,	$d(j, P) \le \delta d(j, Q) + \frac{\thresh - \delta}{\thresh - 1} d(P, Q)$.

	 \hspace{2.1cm} $\forall j \in B_1$, for any $\delta > \thresh$,	$d(j, P) \le \delta d(j, Q) - \frac{\delta - \thresh}{\thresh + 1} d(P, Q)$.
\end{lemma}

\begin{proof}
	First consider the case that $1 \le \delta \le \thresh$.\\

	$\forall j \in B_1$, $d(j, P) \le \thresh d(j, Q)$. Also, by the triangle inequality, $d(j, P) \le d(j, Q) + d(P, Q)$. By a linear combination of these two inequalities,

		\begin{align*}
		d(j, P) &\le \frac{\delta - 1}{\thresh - 1} \thresh d(j, Q) + (1- \frac{\delta - 1}{\thresh - 1})(d(j, Q) + d(P, Q)) \\
		&\le \delta d(j, Q) + \frac{\thresh - \delta}{\thresh - 1} d(P, Q)
	  \end{align*}

	Then consider the case that $\delta > \thresh$.\\

 $\forall j \in B_1$, $d(j, P) \le \thresh d(j, Q)$. By the triangle inequality,

	\begin{align*}
	d(P,Q) \le d(j, P) + d(j, Q) \le \thresh d(j, Q) + d(j, Q) = (1 + \thresh) d(j, Q)
	\end{align*}

	Thus $d(j, Q) \ge \frac{1}{1 + \thresh} d(P, Q)$. $\forall j \in B_1$,

  \begin{align*}
		d(j, P) &\le \thresh d(j, Q) \\
						&= \delta d(j, Q) - (\delta - \thresh) d(j, Q) \\
						&\le \delta d(j, Q) - \frac{\delta - \thresh}{\thresh + 1} d(P, Q)
	\end{align*}

\end{proof}

\begin{lemma}
	\label{lemma-1-thresh-B2}
	 $\forall j \in B_2$, $d(j, P) \le d(j, Q) + d(P, Q)$.
\end{lemma}

\begin{proof}
	This lemma follows directly by the triangle inequality.
\end{proof}

Using the four lemmas above, sum up for all voters, for any $\delta > \thresh$,

\begin{align}
	\label{thm-1-thresh-eq1}
	&\sum_{i \in A_1} d(i, P) + \sum_{i \in A_2} d(i, P) +\sum_{j \in B_1} d(j, P) + \sum_{j \in B_2} d(j, P) \nonumber \\
	&\le \delta \sum_{i \in A_1} d(i, Q) - |A_1| \frac{\delta-1}{2} d(P, Q)
	+ \delta \sum_{i \in A_2} d(i, Q) - |A_2| \frac{\delta \thresh - 1}{\thresh + 1} d(P, Q) \nonumber \\
	&+ \delta \sum_{j \in B_1}d(j, Q) - |B_1| \frac{\delta - \thresh}{\thresh + 1} d(P, Q)
	+ \sum_{j \in B_2}d(j, Q) + |B_2| d(P, Q) \nonumber \\
	&\le \delta \sum_i d(i, Q) + (- |A_1| \frac{\delta-1}{2} - |A_2| \frac{\delta \thresh - 1}{\thresh + 1} - |B_1| \frac{\delta - \thresh}{\thresh + 1} + |B_2|) d(P, Q)
\end{align}

Similarly, for any $1 \le \delta \le \thresh$,
\begin{align}
	\label{thm-1-thresh-eq2}
	\sum_i d(i, P) \le \delta \sum_i d(i, Q) + (- |A_1| \frac{\delta-1}{2} - |A_2| \frac{\delta \thresh - 1}{\thresh + 1} + |B_1| \frac{\thresh-\delta}{\thresh - 1} + |B_2|) d(P, Q)
\end{align}

Now we prove Theorem \ref{thm-1-thresh} by considering two cases: $\thresh \ge \sqrt{2} + 1$ and $\thresh < \sqrt{2} + 1$.\\

  \textbf{Case 1, $\thresh \ge \sqrt{2} + 1$, and $\frac{\thresh+1}{\thresh-1} |A_2| + |A_1| \ge |B_1| + \frac{\thresh+1}{\thresh-1} |B_2|$}\\

   We prove the distortion is at most $\frac{3\thresh - 1}{\thresh + 1}$ in this case. Set $\delta = \frac{3\thresh - 1}{\thresh + 1}$. Note that when $\thresh \ge 1$, $\delta = \frac{3\thresh - 1}{\thresh + 1} \le \thresh$. By inequality \ref{thm-1-thresh-eq2}, if we can prove $(- |A_1| \frac{\delta - 1}{2} - |A_2| \frac{\delta \thresh - 1}{\thresh + 1} + |B_1| \frac{\thresh-\delta}{\thresh - 1} + |B_2|) \le 0$, then $\sum_i d(i, P) \le \delta \sum_i d(i,Q)$. \\

	 When $\delta = \frac{3\thresh - 1}{\thresh + 1}$,

	 \begin{align*}
		 & - |A_1| \frac{\delta-1}{2} - |A_2| \frac{\delta \thresh - 1}{\thresh + 1} + |B_1| \frac{\thresh-\delta}{\thresh - 1} + |B_2| \\
		 &= - \frac{ \thresh - 1}{\thresh + 1}|A_1| - \frac{3\thresh^2 - 2 \thresh - 1}{(\thresh + 1)^2} |A_2| + \frac{ \thresh - 1}{\thresh + 1}|B_1| + |B_2|\\
		 &\le - \frac{ \thresh - 1}{\thresh + 1}|A_1| -  |A_2| + \frac{ \thresh - 1}{\thresh + 1}|B_1| + |B_2|\\
		 &\le 0
	 \end{align*}

	 The second to last line follows because $\frac{3\thresh^2 - 2 \thresh - 1}{(\thresh + 1)^2} \ge 1$ when $\thresh \ge \sqrt{2} + 1$. The last line follows because $\frac{\thresh+1}{\thresh-1} |A_2| + |A_1| \ge |B_1| + \frac{\thresh+1}{\thresh-1} |B_2|$.\\

	\textbf{Case 2, $\thresh < \sqrt{2} + 1$, and $\thresh |A_2| + |A_1| \ge |B_1| + \thresh |B_2|$}\\

	We prove the distortion is at most $\frac{\thresh + 2}{\thresh}$ in this case. Set $\delta = \frac{\thresh + 2}{\thresh}$. Furthermore, we consider two subcases that $1 \le \thresh < 2$ and $2 \le \thresh < \sqrt{2} + 1 $.\\

	\textbf{Case 2.1, $2 \le \thresh < \sqrt{2} + 1$}\\

	When $2 \le \thresh < \sqrt{2} + 1$ and $\delta = \frac{\thresh + 2}{\thresh}$, it is easy to show that $1 \le \delta \le \thresh$. By inequality \ref{thm-1-thresh-eq2}, if we can prove $(- |A_1| \frac{\delta-1}{2} - |A_2| \frac{\delta \thresh - 1}{\thresh + 1} + |B_1| \frac{\thresh - \delta}{\thresh - 1} + |B_2|) \le 0$, then $\sum_i d(i, P) \le \delta \sum_i d(i,Q)$. \\

	When $\delta = \frac{\thresh + 2}{\thresh}$,

	\begin{align*}
		& - |A_1| \frac{\delta-1}{2} - |A_2| \frac{\delta \thresh - 1}{\thresh + 1} + |B_1| \frac{\thresh-\delta}{\thresh - 1} + |B_2| \\
		&= - \frac{1}{\thresh}|A_1| - |A_2| + \frac{\thresh^2 - \thresh -2}{\thresh(\thresh - 1)}|B_1| + |B_2|\\
		&\le - \frac{1}{\thresh}|A_1| - |A_2| + \frac{1}{\thresh}|B_1| + |B_2|\\
		&\le 0
	\end{align*}

  The second to last line follows because $\frac{\thresh^2 - \thresh -2}{\thresh(\thresh - 1)} \le \frac{1}{\thresh}$ when $2 \le \thresh < \sqrt{2} + 1$. The last line follows because $\thresh |A_2| + |A_1| \ge |B_1| + \thresh |B_2|$.\\

	\textbf{Case 2.2, $1 \le \thresh < 2$}\\

	Because $1 \le \thresh < 2$ and $\delta = \frac{\thresh + 2}{\thresh}$, it is easy to show that $\delta > \thresh$. By inequality \ref{thm-1-thresh-eq1}, if we can prove $ (- |A_1| \frac{\delta-1}{2} - |A_2| \frac{\delta \thresh - 1}{\thresh + 1} - |B_1| \frac{\delta - \thresh}{\thresh + 1} + |B_2|) \le 0$, then $\sum_i d(i, P) \le \delta \sum_i d(i,Q)$. When $\delta = \frac{\thresh + 2}{\thresh}$,

	\begin{align*}
		& - |A_1| \frac{\delta-1}{2} - |A_2| \frac{\delta \thresh - 1}{\thresh + 1} - |B_1| \frac{\delta - \thresh}{\thresh + 1} + |B_2| \\
		&= - \frac{1}{\thresh}|A_1| - |A_2| + \frac{\thresh - 2}{\thresh}|B_1| + |B_2|\\
		&\le - \frac{1}{\thresh}|A_1| - |A_2| + \frac{1}{\thresh}|B_1| + |B_2|\\
		&\le 0
	\end{align*}

	The second to last line follows because $\frac{\thresh - 2}{\thresh} < 0 < \frac{1}{\thresh}$ when $1 \le \thresh < 2$. The last line follows because $\thresh |A_2| + |A_1| \ge |B_1| + \thresh |B_2|$.

  Thus, we have shown that the distortion is at most $\frac{3\thresh - 1}{\thresh + 1}$ when $\thresh \ge \sqrt{2} + 1$, and at most $\frac{\thresh + 2}{\thresh}$ when $\thresh < \sqrt{2} + 1$. Note that $\frac{3\thresh - 1}{\thresh + 1} \ge \frac{\thresh + 2}{\thresh}$ when $\thresh \ge \sqrt{2} + 1$, and $\frac{3\thresh - 1}{\thresh + 1} < \frac{\thresh + 2}{\thresh}$ when $\thresh < \sqrt{2} + 1$. Thus, the distortion of the weighted majority rule in this setting is $\max\{\frac{3\thresh - 1}{\thresh + 1}, \frac{\thresh + 2}{\thresh} \}$.
\end{proof}

Note that Weighted Majority Rule \ref{weighted_majority_rule_1_tau} is not the only rule that gives the optimal
distortion for two candidates. Consider the following simpler rule:

\begin{decision_rule}
\label{weighted_majority_rule_1_tau_simple}
 Given voters' preferences and a threshold $\thresh$ for two candidates, assign weight $\frac{\thresh+1}{\thresh-1}$ to all the voters with preference strengths $> \thresh$ and weight $1$ to all the voters with preference strengths $\le \thresh$.
\end{decision_rule}

This rule gives the same distortion as Weighted Majority Rule \ref{weighted_majority_rule_1_tau} for two candidates, as we prove below. When extending these rules to more than 2 candidates, however, Weighted Majority Rule \ref{weighted_majority_rule_1_tau} allows us to form better mechanisms, thus sacrificing a small amount of simplicity for an improvement in distortion. We discuss this in the next section.

\begin{theorem}Weighted Majority Rule \ref{weighted_majority_rule_1_tau_simple} has a distortion of at most $\max\{\frac{3\thresh - 1}{\thresh + 1}, \frac{\thresh + 2}{\thresh} \}$.
\end{theorem}

\begin{proof}
	Denote the set of voters prefer $P$ with preference strengths $>\thresh$ as $A_2$, and with preference strengths $\le \thresh$ as $A_1$. Also denote the set of voters prefer $Q$ with preference strengths $>\thresh$ as $B_2$, and with preference strengths $\le \thresh$ as $B_1$. Without loss of generality, suppose we choose $P$ as the winner by Weighted Majority Rule \ref{weighted_majority_rule_1_tau_simple}. Thus, $\frac{\thresh+1}{\thresh-1} |A_2| + |A_1| \ge |B_1| + \frac{\thresh+1}{\thresh-1} |B_2|$. \\

  Similar to the proof of Theorem \ref{thm-1-thresh}, we discuss three cases based on different values of $\thresh$.\\

	\textbf{Case 1, $1 \le \thresh < 2$}\\

	Set $\delta = \frac{\thresh + 2}{\thresh}$. Because $1 \le \thresh < 2$, it is easy to show that $\delta > \thresh$. By inequality \ref{thm-1-thresh-eq1}, if we can prove $ (- |A_1| \frac{\delta-1}{2} - |A_2| \frac{\delta \thresh - 1}{\thresh + 1} - |B_1| \frac{\delta - \thresh}{\thresh + 1} + |B_2|) \le 0$, then $\sum_i d(i, P) \le \delta \sum_i d(i,Q)$. When $\delta = \frac{\thresh + 2}{\thresh}$,

	\begin{align*}
		& - |A_1| \frac{\delta-1}{2} - |A_2| \frac{\delta \thresh - 1}{\thresh + 1} - |B_1| \frac{\delta - \thresh}{\thresh + 1} + |B_2| \\
		&= - \frac{1}{\thresh}|A_1| - |A_2| + \frac{\thresh - 2}{\thresh}|B_1| + |B_2|\\
		&\le - \frac{\thresh - 1}{\thresh + 1}|A_1| - |A_2| + \frac{\thresh - 1}{\thresh + 1}|B_1| + |B_2|\\
		&\le 0
	\end{align*}

	The second to last line follows because $\frac{1}{\thresh} \ge \frac{\thresh - 1}{\thresh + 1}$ and $\frac{\thresh - 2}{\thresh} < 0 < \frac{\thresh - 1}{\thresh + 1}$ when $1 \le \thresh < 2$. The last line follows because $\frac{\thresh + 1}{\thresh - 1} |A_2| + |A_1| \ge |B_1| + \frac{\thresh + 1}{\thresh - 1} |B_2|$.\\

	\textbf{Case 2, $2 \le \thresh < \sqrt{2} + 1$}\\

	Set $\delta = \frac{\thresh + 2}{\thresh}$. When $2 \le \thresh < \sqrt{2} + 1$, it is easy to show that $1 \le \delta \le \thresh$. By inequality \ref{thm-1-thresh-eq2}, if we can prove $(- |A_1| \frac{\delta-1}{2} - |A_2| \frac{\delta \thresh - 1}{\thresh + 1} + |B_1| \frac{\thresh - \delta}{\thresh - 1} + |B_2|) \le 0$, then $\sum_i d(i, P) \le \delta \sum_i d(i,Q)$. \\

	When $\delta = \frac{\thresh + 2}{\thresh}$,

	\begin{align*}
		& - |A_1| \frac{\delta-1}{2} - |A_2| \frac{\delta \thresh - 1}{\thresh + 1} + |B_1| \frac{\thresh-\delta}{\thresh - 1} + |B_2| \\
		&= - \frac{1}{\thresh}|A_1| - |A_2| + \frac{\thresh^2 - \thresh -2}{\thresh(\thresh - 1)}|B_1| + |B_2|\\
		&\le - \frac{\thresh - 1}{\thresh + 1}|A_1| - |A_2| + \frac{\thresh - 1}{\thresh + 1}|B_1| + |B_2|\\
		&\le 0
	\end{align*}

	The second to last line follows because $\frac{1}{\thresh} \ge \frac{\thresh - 1}{\thresh + 1}$ and $\frac{\thresh^2 - \thresh -2}{\thresh(\thresh - 1)} \le \frac{\thresh - 1}{\thresh + 1}$ when $2 \le \thresh < \sqrt{2} + 1$. The last line follows because $\frac{\thresh + 1}{\thresh - 1} |A_2| + |A_1| \ge |B_1| + \frac{\thresh + 1}{\thresh - 1} |B_2|$.\\

	\textbf{Case 3, $\thresh \ge \sqrt{2} + 1$}\\

	 Set $\delta = \frac{3\thresh - 1}{\thresh + 1}$. Note that when $\thresh \ge \sqrt{2} + 1$, $1 \le \delta \le \thresh$. By inequality \ref{thm-1-thresh-eq2}, if we can prove $(- |A_1| \frac{\delta - 1}{2} - |A_2| \frac{\delta \thresh - 1}{\thresh + 1} + |B_1| \frac{\thresh-\delta}{\thresh - 1} + |B_2|) \le 0$, then $\sum_i d(i, P) \le \delta \sum_i d(i,Q)$. \\

	 When $\delta = \frac{3\thresh - 1}{\thresh + 1}$,

	 \begin{align*}
		 & - |A_1| \frac{\delta-1}{2} - |A_2| \frac{\delta \thresh - 1}{\thresh + 1} + |B_1| \frac{\thresh-\delta}{\thresh - 1} + |B_2| \\
		 &= - \frac{ \thresh - 1}{\thresh + 1}|A_1| - \frac{3\thresh^2 - 2 \thresh - 1}{(\thresh + 1)^2} |A_2| + \frac{ \thresh - 1}{\thresh + 1}|B_1| + |B_2|\\
		 &\le - \frac{ \thresh - 1}{\thresh + 1}|A_1| -  |A_2| + \frac{ \thresh - 1}{\thresh + 1}|B_1| + |B_2|\\
		 &\le 0
	 \end{align*}

	 The second to last line follows because $\frac{3\thresh^2 - 2 \thresh - 1}{(\thresh + 1)^2} \ge 1$ when $\thresh \ge \sqrt{2} + 1$. The last line follows because $\frac{\thresh+1}{\thresh-1} |A_2| + |A_1| \ge |B_1| + \frac{\thresh+1}{\thresh-1} |B_2|$.

	 Thus, we proved that the distortion is at most $\frac{\thresh + 2}{\thresh}$ when $1 \le \thresh < \sqrt{2} + 1$, and at most $\frac{3\thresh - 1}{\thresh + 1}$ when $\thresh \ge \sqrt{2} + 1$.
\end{proof}

%%%%%%
\subsection{Multiple candidates (given preferences and a threshold $\thresh$)}

In this section, we discuss mechanisms with small distortion for multiple ($\geq 3$) candidates. We assume that we are given the ordinal preference ordering of each voter for all the candidates, as well as an indication whether, for every pair of candidates, the voter has a strong preference ($>\thresh$), or a weak preference ($\leq\thresh$). While this certainly requires more than a single bit of information for every voter, we believe that such data is reasonably possible to collect: it is usually easy for users to express whether they prefer option A to option B {\em strongly} or {\em weakly}, as opposed to trying to quantify exactly how strong their preference is. In reality we would need to compare only the obviously front-runner candidates in this way, and would not actually need this thresholded knowledge for {\em every} pair of candidates. As discussed in the Introduction, this information could also be reasonably estimated from other sources, such as the amount of monetary donations, attendance to political rallies, the amount of ``buzz" on social media, etc.

The mechanisms we consider are as follows. First, we create a weighted majority graph by choosing pairwise winners using Majority Rule \ref{weighted_majority_rule_1_tau}. Then we study the distortion of the winner(s) in the uncovered set \cite{moulin1986choosing} in this majority graph. Recall that if a candidate $P$ is in the uncovered set, it means that for any candidate $Z$, either $P$ beats $Z$ directly, or there exists another candidate $Q$ such that $P$ beats $Q$, and $Q$ beats $Z$. The uncovered set is always known to be non-empty, and for example the Copeland mechanism always chooses a candidate in the uncovered set.

We begin with the following useful lemma due to Goel at al. \cite{goel2017metric}

\begin{lemma} \label{Goel} \emph{(Goel et al, 2017)}\\
If a majority of voters prefer $P$ to $Q$, then $SC(P) \leq 2 \cdot SC(Z) + SC(Q)$ for any other possible candidate $Z$.
\end{lemma}

We first show that while this lemma certainly does not hold for all pairwise majority rules, this lemma can be generalized specifically for Majority Rule \ref{weighted_majority_rule_1_tau}. We then use this to prove bounds on the distortion of the above ``uncovered set" mechanisms. This lemma is precisely why we use Majority Rule \ref{weighted_majority_rule_1_tau} instead of, for example, simpler conditions such as Majority Rule \ref{weighted_majority_rule_1_tau_simple}, since while their distortion for two candidates remains the same, the theorem below fails to hold.

\begin{theorem}
\label{thm_PQZ_1_tau}
If Majority Rule \ref{weighted_majority_rule_1_tau} selects P over Q, then $SC(P) \leq 2 \cdot SC(Z) + SC(Q)$ where $Z$ can be any point in the metric space.
\end{theorem}

\begin{proof}
	We use the same notation as before. Let $A_1$ denote a subset of voters that prefer $P$ to $Q$ with preference strengths $\le \thresh$, and let $A_2$ denote a subset of voters that prefer $P$ to $Q$ with preference strengths $> \thresh$. Also $B_1$ denote a subset of voters prefer $Q$ to $P$ with preference strengths $\le \thresh$, and let $B_2$ denote a subset of voters prefer $Q$ to $P$ with preference strengths $> \thresh$. Without loss of generality, suppose we choose $P$ as the winner by our weighted majority rule. It means that if $\thresh \ge \sqrt{2} + 1$, $\frac{\thresh+1}{\thresh-1} |A_2| + |A_1| \ge |B_1| + \frac{\thresh+1}{\thresh-1} |B_2|$, and if $\thresh < \sqrt{2} + 1$,   $\thresh |A_2| + |A_1| \ge |B_1| + \thresh |B_2|$. \\

	From Lemma \ref{Goel}, we know that if $|A_2| + |A_1| \ge |B_1| + |B_2|$, then $SC(P) \leq 2 \cdot SC(Z) + SC(Q)$.  Consider the case that $|A_2| + |A_1| < |B_1| + |B_2|$, it is not possible that $|A_2| < |B_2|$ and $|A_1| >= |B_1|$, because $A_2$ and $B_2$ have heavier weight than $A_1$ and $B_1$. Thus the only case left is $|A_2| >= |B_2|$ and $|A_1| < |B_1|$.

	 We separate the voters in $A_2$ into two subsets $A_2'$ and $A_2 - A_2'$, such that $|A_2'| = |B_2|$. Similarly, we separate the voters in $B_1$ into two subsets $B_1'$ and $B_1 - B_1'$, such that $|B_1'| = |A_1|$.\\

	\textbf{Case 1. $\thresh \ge \sqrt{2} + 1$}\\

	By our weighted majority rule,

	\begin{align*}
	\frac{\thresh + 1}{\thresh - 1}|A_2| + |A_1| &\ge |B_1| + \frac{\thresh + 1}{\thresh - 1} |B_2| \\
	\frac{\thresh + 1}{\thresh - 1}(|A_2'| + |A_2 - A_2'|) + |B_1| &\ge (|B_1'| + |B_1 - B_1'|) + \frac{\thresh + 1}{\thresh - 1} |B_2| \\
	\frac{\thresh + 1}{\thresh - 1}|A_2 - A_2'| &\ge  |B_1 - B_1'| \\
  \end{align*}

  First, bound $\sum\limits_{i \in B_1 - B_1'} d(i,P)$, by the triangle inequality,

	\begin{align*}
     \sum\limits_{i \in B_1 - B_1'} d(i,P) &\le \sum\limits_{B_1 - B_1'} d(i,Z) + \sum\limits_{B_1 - B_1'} d(P, Z)\\
		 &= \sum\limits_{B_1 - B_1'} d(i,Z) + |B_1 - B_1'| \ d(P, Z)\\
		 &\le \sum\limits_{B_1 - B_1'} d(i,Z) + \frac{\thresh + 1}{\thresh - 1}|A_2 - A_2'| \ d(P, Z)\\
		 &= \sum\limits_{B_1 - B_1'} d(i,Z) + \frac{\thresh + 1}{\thresh - 1}\sum\limits_{A_2 - A_2'} d(P, Z)
	\end{align*}

	The second to last line follows because $\frac{\thresh + 1}{\thresh - 1}|A_2 - A_2'| \ge  |B_1 - B_1'|$. Using the triangle inequality again, $\forall i \in A_2 - A_2'$, $d(P,Z) \le d(i,P) + d(i,Z)$, and also note that $d(i, P) \le \frac{1}{\thresh}d(i, Q)$. Thus,

	\begin{align*}
     \sum\limits_{i \in B_1 - B_1'} d(i,P) &\le  \sum\limits_{B_1 - B_1'} d(i,Z) + \frac{\thresh + 1}{\thresh - 1}\sum\limits_{A_2 - A_2'} d(P, Z)\\
		 &\le \sum\limits_{B_1 - B_1'} d(i,Z) + \frac{\thresh + 1}{\thresh - 1}\sum\limits_{A_2 - A_2'} (d(i,P) + d(i,Z)) \\
		 &\le \sum\limits_{B_1 - B_1'} d(i,Z) + \frac{\thresh + 1}{\thresh - 1}\sum\limits_{A_2 - A_2'} (\frac{1}{\thresh}d(i,Q) + d(i,Z)) \\
		 &= \frac{\thresh + 1}{\thresh - 1} \sum\limits_{A_2 - A_1'} d(i,Z) + \sum\limits_{B_1 - B_1'} d(i,Z) + \frac{\thresh + 1}{\thresh(\thresh - 1)}\sum\limits_{A_2 - A_2'} d(i,Q)\\
	\end{align*}

  Multiply both sides by $\frac{\thresh - 1}{\thresh}$,

	\begin{equation}
		\label{thm_PQZ_1_tau_eq1}
		\frac{\thresh - 1}{\thresh} \sum\limits_{i \in B_1 - B_1'} d(i,P) \le \frac{\thresh + 1}{\thresh}\sum\limits_{A_2 - A_1'} d(i,Z) + \frac{\thresh - 1}{\thresh} \sum\limits_{B_1 - B_1'} d(i,Z) + \frac{\thresh + 1}{\thresh ^ 2} \sum\limits_{A_2 - A_2'} d(i,Q)
	\end{equation}

  Also, $\forall i \in B_1 - B_1'$, $d(i, P) \le \thresh d(i, Q)$. So $\sum\limits_{i \in B_1 - B_1'} d(i, P) \le \thresh \sum\limits_{i \in B_1 - B_1'} d(i, Q)$. Divide both sides by $\thresh$, we get $\frac{1}{\thresh} \sum\limits_{i \in B_1 - B_1'} d(i, P) \le \sum\limits_{i \in B_1 - B_1'} d(i, Q)$. Finally, $\forall i \in A_2 - A_2'$, bound $d(i, P)$ by $\frac{1}{\thresh}d(i, Q)$. Together with Inequality \ref{thm_PQZ_1_tau_eq1},

	\begin{align*}
		& \sum\limits_{i \in A_2 - A_2'}d(i,P) + \sum\limits_{i \in B_1 - B_1'}d(i,P) \\
		&\leq \frac{1}{\thresh} \sum\limits_{A_2 - A_2'} d(i,Q) +  \frac{1}{\thresh} \sum\limits_{i \in B_1 - B_1'} d(i, P) + \frac{\thresh - 1}{\thresh} \sum\limits_{i \in B_1 - B_1'} d(i, P)\\
    &\le \frac{1}{\thresh} \sum\limits_{A_2 - A_2'} d(i,Q) + \sum\limits_{i \in B_1 - B_1'} d(i, Q) + \frac{\thresh + 1}{\thresh}\sum\limits_{A_2 - A_1'} d(i,Z) + \frac{\thresh - 1}{\thresh} \sum\limits_{B_1 - B_1'} d(i,Z) + \frac{\thresh + 1}{\thresh ^ 2} \sum\limits_{A_2 - A_2'} d(i,Q)\\
		&= \frac{\thresh + 1}{\thresh}\sum\limits_{A_2 - A_1'} d(i,Z) + \frac{\thresh - 1}{\thresh} \sum\limits_{B_1 - B_1'} d(i,Z) + \frac{2\thresh + 1}{\thresh ^ 2} \sum\limits_{A_2 - A_2'} d(i,Q) + \sum\limits_{i \in B_1 - B_1'} d(i, Q)\\
		&\le 2 (\sum\limits_{i \in A_2 - A_2'} d(i,Z) + \sum\limits_{i \in B_1 - B_1'} d(i,Z)) + \sum\limits_{i \in A_2 - A_2'} (i,Q) + \sum\limits_{i \in B_1 - B_1'} d(i,Q)
	\end{align*}

	The last line follows because $\frac{2\thresh+1}{\thresh^2} \le 1$ when $\thresh \ge \sqrt{2}
 	+1$.\\

  % \begin{align*}
	% &\sum\limits_{i \in A_2 - A_2'}(i,P) + \sum\limits_{i \in B_1 - B_1'}(i,P) \\
	% &\leq \sum\limits_{i \in B_1 - B_1'} (i,Q) + \frac{1}{\thresh} \sum\limits_{A_2 - A_2'}(i,Q) + \frac{\thresh - 1}{\thresh} \sum\limits_{B_1 - B_1'}(i,Z) + \frac{\thresh - 1}{\thresh} |B_1 - B_1'| (P,Z)\\
	% &\leq  \sum\limits_{i \in B_1 - B_1'} (i,Q) + \frac{1}{\thresh} \sum\limits_{A_2 - A_2'}(i,Q) + \frac{\thresh - 1}{\thresh} \sum\limits_{B_1 - B_1'}(i,Z) + \frac{\thresh+1}{\thresh} \sum\limits_{i \in A_2 - A_2'} ((i,P) + (i,Z)) \\
	% &\leq  \sum\limits_{i \in B_1 - B_1'} (i,Q) + \frac{1}{\thresh} \sum\limits_{A_2 - A_2'}(i,Q) + \frac{\thresh - 1}{\thresh} \sum\limits_{B_1 - B_1'}(i,Z) + \frac{\thresh+1}{\thresh^2} \sum\limits_{i \in A_2 - A_2'} (i,Q) + \frac{\thresh+1}{\thresh} \sum\limits_{i \in A_2 - A_2'} (i,Z)\\
	% &\leq \frac{\thresh+1}{\thresh} \sum\limits_{i \in A_2 - A_2'} (i,Z) + \frac{\thresh-1}{\thresh} \sum_{i \in B_1 - B_1'} (i,Z) + \frac{2\thresh+1}{\thresh^2} \sum\limits_{i \in A_2 - A_2'} (i,Q) + \sum\limits_{i \in B_1 - B_1'} (i,Q)\\
	% &\le 2 (\sum\limits_{i \in A_2 - A_2'} (i,Z) + \sum\limits_{i \in B_1 - B_1'} (i,Z)) + \sum\limits_{i \in A_2 - A_2'} (i,Q) + \sum\limits_{i \in B_1 - B_1'} (i,Q)
	% \end{align*}

	\textbf{Case 2. $\thresh < \sqrt{2} + 1$}\\

	By our weighted majority rule,
	\begin{align*}
  \thresh|A_2| + |A_1| &\ge |B_1| + \thresh |B_2| \\
	\thresh(|A_2'| + |A_2 - A_2'|) + |A_1| &\ge (|B_1'| + |B_1 - B_1'|) + \thresh |B_2| \\
	\thresh|A_2 - A_2'|  &\ge |B_1 - B_1'|
	\end{align*}

	The proof is almost the same as \textbf{Case 1}, except that we use the inequality above to bound the ratio between $|B_1 - B_1'|$ and $|A_2 - A_2'|$. Similar to Inequality \ref{thm_PQZ_1_tau_eq1}, we get:

	\begin{equation}
		\label{thm_PQZ_1_tau_eq2}
		\frac{\thresh - 1}{\thresh} \sum\limits_{i \in B_1 - B_1'} d(i,P) \le (\thresh - 1)\sum\limits_{A_2 - A_1'} d(i,Z) + \frac{\thresh - 1}{\thresh} \sum\limits_{B_1 - B_1'} d(i,Z) + \frac{\thresh - 1}{\thresh} \sum\limits_{A_2 - A_2'} d(i,Q)
	\end{equation}

  Then bound $\sum\limits_{i \in A_2 - A_2'}d(i,P) + \sum\limits_{i \in B_1 - B_1'}d(i,P)$ similarly to \textbf{Case 1},

	\begin{align*}
		& \sum\limits_{i \in A_2 - A_2'}d(i,P) + \sum\limits_{i \in B_1 - B_1'}d(i,P) \\
		&= \frac{1}{\thresh} \sum\limits_{A_2 - A_2'} d(i,Q) +  \frac{1}{\thresh} \sum\limits_{i \in B_1 - B_1'} d(i, P) + \frac{\thresh - 1}{\thresh} \sum\limits_{i \in B_1 - B_1'} d(i, P)\\
		&\le \frac{1}{\thresh} \sum\limits_{A_2 - A_2'} d(i,Q) + \sum\limits_{i \in B_1 - B_1'} d(i, Q) + (\thresh - 1)\sum\limits_{A_2 - A_1'} d(i,Z) + \frac{\thresh - 1}{\thresh} \sum\limits_{B_1 - B_1'} d(i,Z) + \frac{\thresh - 1}{\thresh} \sum\limits_{A_2 - A_2'} d(i,Q)\\
		&= (\thresh - 1)\sum\limits_{A_2 - A_1'} d(i,Z) + \frac{\thresh - 1}{\thresh} \sum\limits_{B_1 - B_1'} d(i,Z) +  \sum\limits_{A_2 - A_2'} d(i,Q) + \sum\limits_{i \in B_1 - B_1'} d(i, Q)\\
		&\le 2 (\sum\limits_{i \in A_2 - A_2'} d(i,Z) + \sum\limits_{i \in B_1 - B_1'} d(i,Z)) + \sum\limits_{i \in A_2 - A_2'} (i,Q) + \sum\limits_{i \in B_1 - B_1'} d(i,Q)
	\end{align*}

	% \begin{align*}
	% &\sum\limits_{i \in A_2 - A_2'}(i,P) + \sum\limits_{i \in B_1 - B_1'}(i,P) \\
	% &\leq \sum\limits_{i \in B_1 - B_1'} (i,Q) + \frac{1}{\thresh} \sum\limits_{A_2 - A_2'}(i,Q) + \frac{\thresh - 1}{\thresh} \sum\limits_{B_1 - B_1'}(i,Z) + \frac{\thresh - 1}{\thresh} |B_1 - B_1'| (P,Z)\\
	% &\leq  \sum\limits_{i \in B_1 - B_1'} (i,Q) + \frac{1}{\thresh} \sum\limits_{A_2 - A_2'}(i,Q) + \frac{\thresh - 1}{\thresh} \sum\limits_{B_1 - B_1'}(i,Z) + (\thresh-1) \sum\limits_{i \in A_2 - A_2'} ((i,P) + (i, Z)) \\
	% &\leq  \sum\limits_{i \in B_1 - B_1'} (i,Q) + \frac{1}{\thresh} \sum\limits_{A_2 - A_2'}(i,Q) + \frac{\thresh - 1}{\thresh} \sum\limits_{B_1 - B_1'}(i,Z) + \frac{\thresh-1}{\thresh}  \sum\limits_{i \in A_2 - A_2'} (i,Q) + (\thresh-1) \sum\limits_{i \in A_2 - A_2'} (i,Z)\\
	% &\le 2 (\sum\limits_{i \in A_2 - A_2'} (i,Z) + \sum\limits_{i \in B_1 - B_1'} (i,Z)) + \sum\limits_{i \in A_2 - A_2'} (i,Q) + \sum\limits_{i \in B_1 - B_1'} (i,Q)\\
	% \end{align*}

	We have proved $\sum\limits_{i \in A_2 - A_2' + B_1 - B_1'}(i,P) \le \sum\limits_{i \in A_2 - A_2' + B_1 - B_1'} (i, Q) + 2 \sum\limits_{i \in A_2 - A_2' + B_1 - B_1'} (i, Z)$ for any $\thresh \ge 1$. And because $|A_2'| + |A_1| = |B_1'| + |B_2|$, by Lemma \ref{Goel},

	$$\sum\limits_{i \in A_2' + A_1 + B_1' + B_2}(i,P) \le \sum\limits_{i \in A_2' + A_1 + B_1' + B_2}(i, Q) + 2 \sum\limits_{i \in A_2' + A_1 + B_1' + B_2} (i, Z)$$

	Putting everything together, $\sum_{i} (i, P) \le \sum_{i} (i, Q) + 2 \sum_{i} (i, Z)$.
\end{proof}

Now that we have the above theorem, it is easy to establish distortion bounds based on our weighted majority rule.

\begin{theorem}
	\label{thm_ds_1_tau}
	Suppose a weighted majority graph is formed by using Majority Rule \ref{weighted_majority_rule_1_tau} to choose pairwise winners. The distortion of the uncovered set of this graph is at most $\min \{ \max\{\frac{3\thresh - 1}{\thresh + 1}, \frac{\thresh + 2}{\thresh} \} + 2, \max \{(\frac{3\thresh - 1}{\thresh + 1})^2, (\frac{\thresh + 2}{\thresh})^2 \} \}$ in the multiple candidates setting when given voters' ordinal preferences and a threshold $\thresh$.
\end{theorem}

\begin{proof}

Suppose the optimal candidate is $Z$. By definition, for any candidate $P$ in the uncovered set, either $P$ beats $Z$ directly or there exists a candidate $Q$, that $P$ beats $Q$ and $Q$ beats $Z$. And we know that the distortion between two candidates when one beats the other directly is at most $\max\{\frac{3\thresh - 1}{\thresh + 1}, \frac{\thresh + 2}{\thresh} \}$, so it is straight forward that the distortion is at most $\max\{(\frac{3\thresh - 1}{\thresh + 1})^2, (\frac{\thresh + 2}{\thresh})^2 \}$ for any winner in the uncovered set.

Also, because $Q$ beats $Z$, so $SC(Q) \le \max\{\frac{3\thresh - 1}{\thresh + 1}, \frac{\thresh + 2}{\thresh} \} SC(Z)$. By Theorem \ref{thm_PQZ_1_tau}, we know that $SC(P) \le (\max\{\frac{3\thresh - 1}{\thresh + 1}, \frac{\thresh + 2}{\thresh} \} + 2) SC(Z)$. Thus, we can get a upper bound of distortion for the uncovered set of $\min \{\max\{\frac{3\thresh - 1}{\thresh + 1}, \frac{\thresh + 2}{\thresh} \} + 2, \max\{(\frac{3\thresh - 1}{\thresh + 1})^2, (\frac{\thresh + 2}{\thresh})^2 \} \}$.
\end{proof}

\subsection{Choosing the Best Threshold}
What type of knowledge of the strengths in voter preferences is most useful and advantageous? If you could gather data about voter preferences in different ways, what should you aim for in order to reduce distortion? These are some of the questions which we wish to illuminate in this paper.

Our results in the previous two sections shed some light on these decisions. First, it may be surprising (although it really shouldn't be) that knowing only information about very extreme voters (i.e., $\thresh$ being high) or only about very indecisive voters ($\thresh$ being very close to 1) does not help much when compared to only knowing the voters' ordinal preferences. Our results indicate, however, that the optimal thing to do is to differentiate between candidates with lots of supporters who prefer them at least 2 times to other candidates (or more precisely, at least $1+\sqrt{2}$ times), and candidates which have few such supporters. Our results indicate that by obtaining this information, we can improve the quality of the chosen candidate from a 3-approximation to only a 1.83 approximation (for 2 candidates), and from a 5-approximation to a 3.35-approximation (for $\geq 3$ candidates). This is a huge improvement obtained with relatively little extra cost.

\section{Undecided Voters: working without knowing voter preferences}
\label{sec-1tau}
Suppose there are two candidates \P and \Q and for all voters with preference strength greater than threshold \thresh, we know their preferred candidate. For all other voters we know {\em nothing} about their preferences. This is a strict generalization of the case where we just know voter preferences, since that is the case where $\thresh = 1$. As with the case where we only know preferences, the only reasonable voting rule is to select the candidate preferred by the greatest number of voters, out of those for whom we know preferences. This represents the case where voters abstain if their preference strength is not sufficiently high for them to be motivated enough to vote. In this section we consider mechanisms to deal with such undecided or unmotivated voters.

\begin{decision_rule} \label{rule_majority_1thresh}
Given candidates \P and \Q and any single threshold \thresh $\geq$ 1, give all voters with preference strength at least \thresh a weight of 1 and all other voters a weight of 0. Select the candidate by weighted majority rule.
\end{decision_rule}

\begin{theorem}
With two candidates and only the preferences of voters with preference strength greater than \thresh, Weighted Majority Rule \ref{rule_majority_1thresh} achieves a distortion of $\max\{\frac{\thresh+2}{\thresh}, \thresh \}$, and no deterministic mechanism can do better.
\end{theorem}

\begin{proof}
%Suppose, without loss of generality, that \P is the winner. For all voters, consider their individual ratio of $\frac{d(i,P)}{d(i,Q)}$, regardless of which candidate they prefer. For voters who prefer \P this is their preference strength, and for voters who prefer \Q this is the reciprocal of their preference strength. If for all voters this were less than \ds, then clearly we have a distortion of at most \ds by just summing them up. However, for some voters this ratio is higher and for others it is lower. If we think of charging $SC(P)$ to $SC(Q)$, for voters for whom this ratio is lower than \ds we effectively \emph{save} charge, which we can then use to charge the contributions of voters for whom this ratio is higher than \ds. Clearly, for any voters who prefer \P this ratio is less than 1 and so it is less than \ds. For voters who prefer \Q, some voters with weak preferences will allow us to save charge while others with stronger preferences will use up the extra charge. In our proof, we use $d(P,Q)$ as a sort of numeraire or store of value. We first perform the charging for all voters for whom this ratio is small, and we use $d(P,Q)$ to quantify how much extra charge saved. We then show that this quantity of charge stored in terms of $d(P,Q)$ is sufficient to expend the charge from the remaining voters, yielding a distortion at most \ds.\\
The proof is similar to that of Theorem \ref{thm-1-thresh}, once again using $d(P,Q)$ as an intermediate value to charge possible voter distances to. Let $A$ be the set of voters who strongly prefer \P. That is, $A = \{i : \frac{d(i,Q)}{d(i,P)} \geq \thresh \}$. Similarly, define the set of voters who strongly prefer \Q as $B = \{j : \frac{d(j,P)}{d(j,Q)} \geq \thresh \}$. Let the set of remaining voters, whose preference strengths are weaker than \thresh be denoted $C$. Without loss of generality, let \P be the winner over \Q because $|A| \geq |B|$.

\begin{lemma} \label{lemma_1T_A}
$\forall i \in A$, \emph{for any} $\ds \geq 1$ : $d(i,P) \leq \ds d(i,Q) - \frac{\ds \thresh - 1}{\thresh+1} d(P,Q)$.
\end{lemma}

\begin{proof}
$\forall i \in A$ we know that $d(i,P) \leq \frac{1}{\thresh} d(i,Q)$.\\

It follows from triangle inequality that $d(P,Q) \leq d(i,P) + d(i,Q) \leq \frac{1}{\thresh} (i,Q) + (i,Q) = \frac{\thresh+1}{\thresh} (i,Q)$.\\

For any $\ds \geq 1$ we when have

\begin{align*}
d(i,P) &\leq \frac{1}{\thresh} d(i,Q)\\
&= \ds d(i,Q) - (\ds - \frac{1}{\thresh}) d(i,Q)\\
&\leq \ds d(i,Q) - (\ds - \frac{1}{\thresh}) (\frac{\thresh}{\thresh+1}) d(P,Q)\\
&= \ds d(i,Q) - (\frac{\ds \thresh - 1}{\thresh+1}) d(P,Q)
\end{align*}
\end{proof}

Recall that $d(j,P) \leq d(j,Q) + d(P,Q)$ from triangle inequality. It therefore follows from Lemma \ref{lemma_1T_A} that for any $\ds \geq 1$, $$\sumL_{i \in A} d(i,P) + \sumL_{j \in B} (j,P) \leq \ds \sumL_{i \in A} d(i,Q) - |A|(\frac{\ds \thresh - 1}{\thresh+1}) d(P,Q) + \sumL_{j \in B} d(j,Q) + |B| d(P,Q)$$

Let $\ds = \max\{\frac{\thresh+2}{\thresh}, \thresh \}$. We consider the two cases in which either of the two terms in this bound are the larger term.\\

\textbf{Case 1:} If $\thresh \geq 2$ then $\ds = \thresh$, and therefore

\begin{align*}
\sumL_{i \in A} d(i,P) + \sumL_{j \in B} (j,P) &\leq \ds \sumL_{i \in A} d(i,Q) - |A|(\thresh - 1) d(P,Q) + \sumL_{j \in B} d(j,Q) + |B| d(P,Q)\\
&\leq \ds \sumL_{i \in A} d(i,Q) + \sumL_{j \in B} (j,Q) \qquad \text{because \ } |A| \geq |B|.
\end{align*}

\textbf{Case 2:} If $\thresh < 2$ then $\ds = \frac{\thresh+2}{\thresh}$, and therefore

\begin{align*}
\sumL_{i \in A} d(i,P) + \sumL_{j \in B} (j,P)& \leq \ds \sumL_{i \in A} d(i,Q) - |A| d(P,Q) + \sumL_{j \in B} d(j,Q) + |B| d(P,Q)\\
&\leq \ds \sumL_{i \in A} d(i,Q) + \sumL_{j \in B} (j,Q)  \qquad \text{because \ } |A| \geq |B|.
\end{align*}

Lastly, we can see that this upper bound on $\ds$ is tight due to the lower bounds given by examples in Lemma \ref{lem_smallest_lowerbound} and Lemma \ref{lem_largest_lowerbound}.
\end{proof}

\subsection{Choosing the Best Threshold}
If we can only select a single threshold for voter preference strengths, which should we choose? Intuitively, this is analogous to determining how difficult it should be to vote. If it takes a little bit of effort to vote, then you know that the voters who actually do participate have a significant interest in the outcome. However, if the barriers to voting are too high, then the outcome can be decided by a small fraction of the voters and fails to capture their collective preferences as a whole (see Figure \ref{fig:General_Case_Distortion_Plot}). In our setting the optimal choice of threshold is $\argmin\limits_{\thresh} \{ \max\{\frac{\thresh+2}{\thresh}, \thresh \} \} = 2$, yielding a distortion of 2 (instead of 3 for the case when $\tau=1$).

\begin{figure}[htb]
\begin{center}
\includegraphics[scale = 0.5]{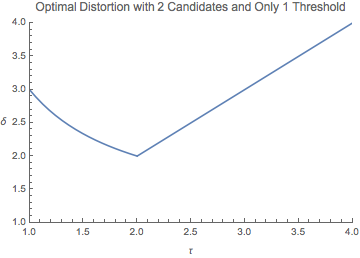}\\
\end{center}
\caption{Best achievable distortion for a single threshold $\thresh$.}
\label{fig:General_Case_Distortion_Plot}
\end{figure}

\subsection{Multiple Candidates (given only a threshold $\thresh$)}

When there are more than two candidates, we study the distortion of the uncovered set.

\begin{theorem}
With mutiple candidates and only the preferences of voters with preference strength greater than \thresh, if Weighted Majority Rule \ref{rule_majority_1thresh} is used to choose pairwise winners, then the distortion of the uncovered set of this graph is at most $\max\{(\frac{\thresh+2}{\thresh})^2, \thresh ^ 2 \}$.
\end{theorem}

\begin{proof}
	Suppose the optimal candidate is $Z$. By definition, for any candidate $P$ in the uncovered set, either $P$ beats $Z$ directly or there exists a candidate $Q$, such that $P$ beats $Q$ and $Q$ beats $Z$. And we know that the distortion between two candidates when one beats the other directly is at most $\max\{\frac{\thresh+2}{\thresh}, \thresh \}$, so it is straight forward that the distortion is at most $\max\{(\frac{\thresh+2}{\thresh})^2, \thresh ^ 2 \}$ for any winner in the uncovered set.
\end{proof}

Note that, unlike in Theorem \ref{thm_ds_1_tau}, for this setting we have to settle for the trivial bound of squaring the distortion for $\geq 3$ candidates. This is because, unlike for the case with known preferences and a threshold, the property that $SC(P) \leq 2 \cdot SC(Z) + SC(Q)$ (Theorem \ref{thm_PQZ_1_tau}) does not hold anymore. Consider the following example: there are three candidates $P$, $Q$, and $Z$, and there is only one voter $i$, that has a preference strength $< \thresh$ between any pair of candidates, so we have no information whatsoever about voter preferences. Without loss of generality, suppose we choose $P$ as the winner. The actual distances could be: $d(i, P) = \thresh - \epsilon$, $d(i, Q) = 1$, and $d(i, Z) = 1$. As $\epsilon$ approaches $0$, $SC(P) \approx \thresh SC(Q)$, and also $SC(P) \approx \thresh SC(Z)$. When $\thresh$ is large, it is not possible to have $SC(P) \le 2 SC(Z) + SC(Q)$. Thus, we cannot bound $SC(P)$ in the multiple candidates setting by $SC(P) \leq 2 \cdot SC(Z) + SC(Q)$ as in Section \ref{sec-1-tau}.

%in this setting, suppose $P$ pairwise beats $Q$, then for any candidate $Z$, the property that $SC(P) \leq 2 \cdot SC(Z) + SC(Q)$ (Theorem \ref{thm_PQZ_1_tau}) in the setting that we are given voter's preferences and a threshold $\thresh$ does not hold anymore. Consider the following example: there are three candidates $P$, $Q$ and $Z$, and there is only one voter $i$, that has a preference strength $< \thresh$ between $P$ and $Q$, so $P$ and $Q$ have the same weight. Suppose we choose $P$ as the winner. The actually distances are: $d(i, P) = \thresh - \epsilon$, $d(i, Q) = 1$ and $d(i, Z) = 0$. $\epsilon$ is a very small positive number. When $\epsilon$ approaches $0$, $SC(P) \approx \thresh SC(Q)$, and because $SC(Z) = 0$, it is not possible to have $SC(P) \le \lambda SC(Z) + SC(Q)$ for arbitrarily large $\lambda$. In other words, Weighted Majority Rule \ref{rule_majority_1thresh} is not $\lambda$-bounded for any constant $\lambda$. Thus, we cannot bound $SC(P)$ in the multiple candidates setting by $SC(P) \leq 2 \cdot SC(Z) + SC(Q)$ as in Section \ref{sec-1-tau}.

%%%%%%%%%%%%%%%%%%%%%%%%%%%%%%%%%%%
%\pagebreak
\section{Distortion with General Thresholds} \label{sec_general}
In this section we generalize some of our results in the previous sections to deal with general preference strength thresholds. We are given thresholds $\{1 \leq \thresh_1 < \thresh_2 < \ldots < \thresh_m\}$, and for every voter $i$ and pair of candidates \P and \Q we know the pair of thresholds between which the preference strength of $i$ falls into. In other words, the more thresholds we have, the less coarse our knowledge of voters preferences. We believe it is realistic to assume that we have one or two, perhaps three, such thresholds, and for most candidate pairs we can create a profile describing how devoted and fanatical their supporters are with respect to these thresholds. However, in this section we consider general sets of thresholds in order to provide bounds on distortion which are as general as possible. For convenience, we let $\thresh_{m+1} = \infty$ and $\thresh_0 = \frac{1}{\thresh_1}$.

We begin as before, by analyzing the case with only 2 candidates \P and \Q, and then extending our results to multiple candidates.

\begin{condition} \label{cond_general}
Let $\ds = \max\limits_{0 \leq l \leq m} \{ \frac{\threshl \threshup + 2\threshup - 1}{\threshl \threshup + 1} \}$. Find $k$ such that $\thresh_k \leq \ds < \thresh_{k+1}$. \P wins only if $\sumL_{l=k}^m (\frac{\threshup - \ds}{\threshup - 1})|B_l| \leq \sumL_{l=1}^m (\frac{\ds \threshl - 1}{\threshl + 1})|A_l| + \sumL_{l=1}^{k-1} |B_l| (\frac{\ds - \threshup}{\threshup + 1})$ and \Q wins only if $\sumL_{l=k}^m (\frac{\threshup - \ds}{\threshup - 1})|A_l| \leq \sumL_{l=1}^m (\frac{\ds \threshl - 1}{\threshl + 1})|B_l| + \sumL_{l=1}^{k-1} |A_l| (\frac{\ds - \threshup}{\threshup + 1})$.\\
\end{condition}

The above is not a specific voting rule, but is instead a set of voting rules. We prove below that any rule obeying the above condition has distortion at most $\ds$, and that we can always form a rule satisfying this condition. Note that such a value of $0 < k \leq m$ always exists because distortion is at least $\thresh_1$ (since taking the term for $l=0$ gives $\thresh_1$). It may be that $k = m$, where $\thresh_m \leq \ds$.

\begin{theorem}
Any single-winner voting rule over two candidates which satisfies Condition \ref{cond_general} has distortion $\ds = \max\limits_{0 \leq l \leq m} \{ \frac{\threshl \threshup + 2\threshup - 1}{\threshl \threshup + 1} \}$ and no deterministic mechanism can do better.
\end{theorem}

\begin{proof} $ $\\
{\em Outline:}
First, we prove the upper bound on distortion. We want to show that if \P wins then $SC(P) = \sumL_{l=1}^m \sumL_{i \in A_l} d(i,P) + \sumL_{l=1}^m \sumL_{j \in B_l} d(j,P) + \sumL_{k \in C} d(k,P) \leq \ds \big( \sumL_{l=1}^m \sumL_{i \in A_l} d(i,Q) + \sumL_{l=1}^m \sumL_{j \in B_l} d(j,Q) + \sumL_{k \in C} d(j,Q) \big) = \ds SC(Q)$. We prove this by using four lemmas which each establish an upper bound on the social cost accrued to \P by a subset of the voters. To do this we use $d(\P,\Q)$ as a sort of numeraire or store of value. Summing over the three inequalities in these lemmas proves the upper bound on distortion as long as Condition \ref{cond_general} is met. Tightness follows from Theorem \ref{thm_mT_lowerbound} in Section \ref{sec_lb}.

%Voters with weak preferences in C
\begin{lemma}
If \P wins then
$\sumL_{k \in C} d(k,P) \leq \sumL_{k \in C} \ds d(k,Q)$
\end{lemma}

\begin{proof} $ $\\
$\forall k \in C$ : $d(k,P) \leq \thresh_1 \ d(k,Q)$ and we know $\thresh_1 \leq \ds$ by our choice of $\ds$.
\end{proof}

%Voters in favor of P
\begin{lemma}
If \P wins then
$\sumL_{l = 1}^m \sumL_{i \in A_l} d(i,P) \leq \sumL_{l = 1}^m \sumL_{i \in A_l} \ds d(i,Q) - \sumL_{l = 1}^m \sumL_{i \in A_l} ( \frac{\ds \thresh_l - 1}{\thresh_l + 1}) d(P,Q)$
\end{lemma}

\begin{proof} $ $\\
Recall from the definition of $A_l$ that $\forall l \leq m, \forall i \in A_l$ : $d(i,P) \leq \frac{1}{\thresh_l}d(i,Q)$.\\

This implies $\forall l \leq m, \forall i \in A_l$ : $d(P,Q) \leq d(i,P) + d(i,Q) \leq \frac{\thresh_l + 1}{\thresh_l} d(i,Q)$.\\

It follows that
\begin{align*}
\sumL_{l = 1}^m \sumL_{i \in A_l} d(i,P) &\leq \sumL_{l = 1}^m \sumL_{i \in A_l} \frac{1}{\thresh_l} d(i,Q)\\
&= \sumL_{l = 1}^m \sumL_{i \in A_l} \Big(\ds d(i,Q) - (\ds - \frac{1}{\thresh_l}) d(i,Q) \Big)\\
&\leq \sumL_{l = 1}^m \sumL_{i \in A_l} \Big( \ds d(i,Q) - \big(\frac{\ds \thresh_l - 1}{\thresh_l}\big)\big(\frac{\thresh_l}{\thresh_l + 1}\big) d(P,Q) \Big)\\
&= \sumL_{l = 1}^m \sumL_{i \in A_l} \ds d(i,Q) - \sumL_{l = 1}^m \sumL_{i \in A_l} \Big(\frac{\ds \thresh_l - 1}{\thresh_l + 1} \Big)d(P,Q)
\end{align*}

\end{proof}

\vspace{5mm}
%Voters in favor of Q with weak preferences
\begin{lemma}
If \P wins then $\sumL_{l = 1}^{k-1} \sumL_{j \in B_l} d(j, P) \leq \sumL_{l = 1}^{k-1} \sumL_{j \in B_l} \ds d(j,Q) - \sumL_{l = 1}^{k-1} \sumL_{j \in B_l} ( \frac{\ds - \threshup}{\threshup + 1}) d(P,Q)$
\end{lemma}

\begin{proof}
Recall from the definition of $B_l$ that $\forall l < k, \forall j \in B_l$ : $d(j,P) \leq \threshup d(j,Q)$.\\

This implies $\forall l < k, \forall j \in B_l$ : $d(P,Q) \leq d(j,P) + d(j,Q) \leq (\threshup + 1) d(j,Q)$.\\

It follows that
\begin{align*}
\sumL_{l = 1}^{k-1} \sumL_{j \in B_l} d(j,P) &\leq \sumL_{l = 1}^k \sumL_{j \in B_l} \threshup d(j,Q)\\
&= \sumL_{l = 1}^{k-1} \sumL_{j \in B_l} \Big(\threshup d(j,Q) + (\ds - \threshup) d(j,Q) - (\ds - \threshup) d(j,Q) \Big)\\
&= \sumL_{l = 1}^{k-1} \sumL_{j \in B_l} \ds d(j,Q) -  \sumL_{l = 1}^k \sumL_{j \in B_l} (\ds - \threshup) d(j,Q)\\
&\leq \sumL_{l = 1}^{k-1} \sumL_{j \in B_l} \ds d(j,Q) - \sumL_{l = 1}^k \sumL_{j \in B_l} \Big(\frac{\ds - \threshup}{\threshup + 1} \Big) d(P,Q)
\end{align*}
\end{proof}

\vspace{5mm}
%Voters in favor of Q with strong preferences
\begin{lemma}
If \P wins then $\sumL_{l = k}^{m} \sumL_{j \in B_l} d(j, P) \leq \sumL_{l = k}^{m} \sumL_{j \in B_l} \ds d(j,Q) +  \sumL_{l = k}^{m} \sumL_{j \in B_l} ( \frac{ \threshup - \ds}{\threshup - 1}) d(P,Q)$
\end{lemma}

\begin{proof}
Recall from the definition of $B_l$ that $\forall l \geq k, \forall j \in B_l$ : $d(j,P) \leq \threshup d(j,Q)$.\\

From triangle inequality $\forall j : d(j,P) \leq d(j,Q) + d(P,Q)$.\\

Together these imply, $\forall l \geq k, \forall j \in B_l$ : $d(j,P) \leq x \threshup d(j,Q)  + (1-x) \big( d(j,Q) + d(P,Q) \big)$ for any $0 \leq x \leq 1.$\\

Below, for each $l \geq k$ we choose $x = \frac{\ds - 1}{\threshup - 1} \leq 1$.\\

It follows that
\begin{align*}
\sumL_{l = k}^m \sumL_{j \in B_l} d(j,P) &\leq \sumL_{l = k}^m \sumL_{j \in B_l} \big(\frac{\ds - 1}{\threshup - 1} \big)\threshup d(j,Q)  + \big(1- \frac{\ds - 1}{\threshup - 1} \big) \big( d(j,Q) + d(P,Q) \big)\\
&= \sumL_{l = k}^m \sumL_{j \in B_l} \ds d(j,Q) +  \sumL_{l = k}^m \sumL_{j \in B_l} \big(1- \frac{\ds - 1}{\threshup - 1} \big) d(P,Q)\\
&= \sumL_{l = k}^m \sumL_{j \in B_l} \ds d(j,Q) +  \sumL_{l = k}^m \sumL_{j \in B_l} \big(\frac{\threshup - \ds}{\threshup - 1} \big) d(P,Q)
\end{align*}

\end{proof}

By summing over the inequalities in the four preceding lemmas, we have
$$SC(P) \leq \ds SC(Q) +  d(P,Q) \bigg(\sumL_{l = k}^m \sumL_{j \in B_l} \big(\frac{\threshup - \ds}{\threshup - 1} \big) - \sumL_{l = 1}^k \sumL_{j \in B_l} \Big(\frac{\ds - \threshup}{\threshup + 1} \Big) - \sumL_{l = 1}^m \sumL_{i \in A_l} \Big(\frac{\ds \thresh_l - 1}{\thresh_l + 1} \Big) \bigg)$$

If Condition \ref{cond_general} holds when \P wins the $d(P,Q)$ term on the RHS is non-positive, and we have $SC(P) \leq \ds SC(Q)$ as desired.
\end{proof}

We have now shown that any voting rule obeying the above condition has distortion at most \ds. We now prove that for any instance, selecting one of the two candidates {\em must} satisfy Condition \ref{cond_general}, so we can construct resolute single-winner voting rules which satisfy this condition. Last, we provide a specific weighted majority rule which always satisfies Condition \ref{cond_general}.

\begin{lemma}
Given any instance, i.e., a set of voters, two candidates, and a set of thresholds, selecting at least one of the candidates must satisfy Condition \ref{cond_general}.
\end{lemma}

\begin{proof}
Put another way, at least one of the two inequalities in Condition \ref{cond_general} must hold, so there can be no instance in which neither candidate can be selected.\\

Suppose $\ds \geq \max\limits_{0 \leq l \leq m} \{ \frac{\threshl \threshup + 2\threshup - 1}{\threshl \threshup + 1} \}$.\\

By moving over the denominator, this can be rewritten as $$\forall l \leq m : \ds(\threshl \threshup + 1) \geq \threshl \threshup + 2\threshup - 1$$

or  $$\forall l \leq m : (\ds \threshl - 1)(\threshup - 1) - (\threshl + 1)(\threshup - \ds) \geq 0.$$

We can divide both sizes to obtain $$\forall l \leq m : \frac{(\ds \threshl - 1)(\threshup - 1) - (\threshl + 1)(\threshup - \ds)}{(\threshl + 1)(\threshup - 1)} \geq 0$$

and simplify to get, $$\forall l \leq m : \frac{\ds \threshl - 1}{\threshl + 1} - \frac{\threshup - \ds}{\threshup - 1} \geq 0.$$

We can now express our inequality in terms of the sets of voters $$\sumL_{l = k}^m (|A_l| + |B_l|)(\frac{\ds \threshl - 1}{\threshl + 1} - \frac{\threshup - \ds}{\threshup - 1}) \geq 0$$

and separate to yield $$\sumL_{l = k}^m (|A_l| + |B_l|)(\frac{\ds \threshl - 1}{\threshl + 1} - \frac{\threshup - \ds}{\threshup - 1}) + \sumL_{l = 1}^{k-1} (|A_l| + |B_l|)(\frac{\ds \threshl - 1}{\threshl + 1} + \frac{\ds - \threshup}{\threshup + 1}) \geq 0.$$

We can separate terms further to see that $$\sumL_{l=k}^m (\frac{\threshup - \ds}{\threshup - 1})(|A_l| + |B_l|) \leq \sumL_{l=1}^m (\frac{\ds \threshl - 1}{\threshl + 1})(|A_l|+|B_l|) + \sumL_{l=1}^{k-1} (|A_l| + |B_l|) (\frac{\ds - \threshup}{\threshup + 1}).$$

As a consequence, one of the following must be true for the sum of these inequalities to be true:
\begin{align}
\sumL_{l=k}^m (\frac{\threshup - \ds}{\threshup - 1})|B_l| &\leq \sumL_{l=1}^m (\frac{\ds \threshl - 1}{\threshl + 1})|A_l| + \sumL_{l=1}^{k-1} |B_l| (\frac{\ds - \threshup}{\threshup + 1})\\
\sumL_{l=k}^m (\frac{\threshup - \ds}{\threshup - 1})|A_l| &\leq \sumL_{l=1}^m (\frac{\ds \threshl - 1}{\threshl + 1})|B_l| + \sumL_{l=1}^{k-1} |A_l| (\frac{\ds - \threshup}{\threshup + 1})
\end{align}

\end{proof}

Therefore resolute single-winner voting rules which maintain Condition \ref{cond_general} can be created, and such a rule achieves optimal distortion between two candidates. We consider one such rule below, although many are possible.

\begin{decision_rule} \label{rule_general} $ $\\ %\textbf{General Weighted Majority Rule}\\
For all $l <  k$, assign to all voters in $A_l$ and $B_l$ a weight of $\frac{(\ds+1)(\threshl \threshup - 1)}{(\threshl + 1)(\threshup + 1)}$. For all $l \geq k$, assign voters in $A_l$ and $B_l$ a weight of $\big((\frac{\threshup - \ds}{\threshup - 1}) + (\frac{\ds \threshl - 1}{\threshl+1}) \big)$. Lastly, assign all voters in $C$ a weight of 0. Choose the candidate by a weighted majority vote.
\end{decision_rule}

\begin{theorem}
	\label{thm-distortion-general-rule}
Weighted Majority Rule \ref{rule_general} satisfies Condition \ref{cond_general}, and therefore achieves the optimal distortion for two candidates with preference strength information.
\end{theorem}

\begin{proof}
Consider the two inequalities in Condition \ref{cond_general} which dictate whether it is permissible to choose \P or \Q respectively. We can take the difference RHS - LHS of each inequality, which must be non-negative for at least one of them, and choose the candidate corresponding to the inequality that yields a bigger difference. This is exactly our weighted majority rule.
\end{proof}

Weighted Majority Rule \ref{rule_general} is well-behaved because voters with weaker preferences are assigned smaller weights. Voters whose preferences are so weak that we cannot determine their preferred candidate must have a weight of 0 because it is unknown who they support, and no voters have negative weight. However, voters with preference strength tending towards infinity cannot have infinitely large weights. Here, the weights of the voters whose decisiveness is higher than $\thresh_m$ is $1 + \frac{\ds \thresh_m - 1}{\thresh_m + 1}$, which converges asymptotically to $\ds + 1$ as $\thresh_m \rightarrow \infty$. However, many other rules with the same distortion are possible and it is an open question to determine which rules yield the best distortion for multiple candidates.

How much effort, time, and money, should someone charged with developing a voting protocol, or with choosing an alternative minimizing social cost, spend in order to understand the preference strengths of voters in more detail? With only ordinal preferences $(m = 1, \thresh = 1)$, the best distortion achievable is by simple majority vote, yielding a distortion of 3. However, if we are permitted any single threshold of our choice $(m = 1, 1 < \thresh)$, we can bring the distortion down to significantly to 2. With any two thresholds of our choice $(m = 1, 1 \leq \thresh_1 < \thresh_2)$, we can bring distortion down further to $\nicefrac{5}{3} \approx 1.67$, and as the number of thresholds permitted increases we see distortion converge to $\sqrt{2} \approx 1.4$. (See Figure \ref{fig:General_Case_Distortion_Plot}.) This is because in the limit when we know the exact preference strengths of all voters, distortion can be bounded by $\sqrt{2}$, as we show in the next section. Thus, there is not much incentive to spend a huge amount of money to understand exact preference strengths, as one or two carefully chosen thresholds already provide very good distortion.

For the general case with arbitrary thresholds and no extra assumptions, we can demonstrate a bound of $\ds^2$ on distortion for three or more candidates. This is obtained simply by forming a pairwise majority graph based on the above weighted majority rule, and then taking any alternative in the uncovered set of the resulting graph. It remains an open question whether there exist weighted majority rules that can improve the bound on distortion in the general case using this method, as we can when we have a single threshold and preferences, or preferences alone.  More generally, it is unknown how to get a tight bound on distortion with multiple candidates using any rule, even in the simpler case with only ordinal preferences \cite{mungala2019improved}.

\begin{figure}[htb]
\centering
\begin{minipage}{.43\linewidth}
  \includegraphics[width=\linewidth]{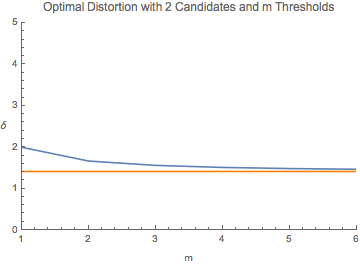}
  \captionof{figure}{Best achievable distortion for two candidates if allowed the best choice of $m$ thresholds. Converges to $\sqrt{2}$ with the number of thresholds.}
  \label{fig:two_candidates_m_thresh}
\end{minipage}
\hspace{.04\linewidth}
\begin{minipage}{.43\linewidth}
  \includegraphics[width=\linewidth]{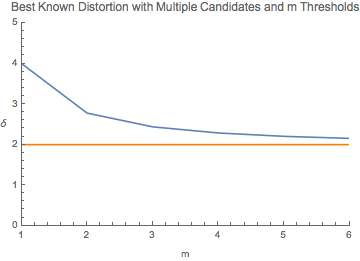}
  \captionof{figure}{Best known distortion for multiple candidates if allowed the best choice of $m$ thresholds. Converges to $2$ with the number of thresholds.}
  \label{fig:multi_candidates_m_thresh}
\end{minipage}
\end{figure}

%%%%%%%%%%%%%%%%%%%%%%%%%%%%%%%%%%%%%%%%%%%%%%%%%%%%
%\pagebreak
\subsection{Exact Preference Strengths of All Voters}

%\subsection{Distortion with 2 Candidates}

In this section, for completeness of analysis, we consider the case when we know the exact preference strengths of all the voters with respect to every pair of candidates. This corresponds to the limit settings in which are have an infinite number of thresholds which includes every number greater than 1. As we established previously, even with this knowledge it is not possible to form deterministic algorithms with distortion better than $\sqrt{2}$. Here we give a mechanism which obtains this bounds.

Suppose there are two candidates $P$ and $Q$, and we are given the preference strengths of every voter. Denote $A$ as the set of voters that prefer $P$ to $Q$, and $B$ as the set of voters that prefer $Q$ to $P$. The preference strength of any $i \in A$ is denoted as $\alpha_i$, and the preference strength of any $j \in B$ is denoted as $\beta_j$,

% \textbf{Picking Majorities Rule}\\
% If $\sum_{i \in A| \alpha_i > \sqrt{2}} (\frac{\sqrt{2} \alpha_i-1}{\alpha_i + 1} + \frac{\alpha_i-\sqrt{2}}{\alpha_i - 1}) + \sum_{i \in A| \alpha_i \le \sqrt{2}} (\frac{(\sqrt{2}+1) (\alpha_i-1)}{\alpha_i + 1}) \ge \sum_{j \in B | \beta_j > \sqrt{2}}  (\frac{\sqrt{2} \beta_j-1}{\beta_j + 1} + \frac{\beta_j-\sqrt{2}}{\beta_j - 1}) + \sum_{j \in B| \beta_j \le \sqrt{2}} (\frac{(\sqrt{2}+1) (\beta_j-1)}{\beta_j + 1})$, we choose $P$ as the winner, otherwise we choose $Q$ as the winner.\\

\begin{theorem}
	\label{thm-exact}
With 2 candidates $P$ and $Q$ in a metric,  given the exact preference strength of every voter, if $\sum_{i \in A} \frac{\sqrt{2} \alpha_i-1}{\alpha_i + 1} \ge \sum_{j \in B | \beta_j > \sqrt{2}}  \frac{\beta_j - \sqrt{2}}{\beta_j - 1} - \sum_{j \in B | \beta_j \le \sqrt{2}} \frac{(\sqrt{2} - \beta_j)}{\beta_j + 1}$, then $SC(P) \le \sqrt{2} SC(Q)$. \\
\end{theorem}

\begin{proof}

	$\forall i \in A$,
	$$d(P, Q) \le d(i, P) + d(i, Q) = \frac{\alpha_i + 1}{\alpha_i} (i, Q)$$

	Bound the sum of $d(i, P)$ for all $i \in A$:

	\begin{align*}
		\sum_{i \in A} d(i, P) &= \sum_{i \in A} \frac{1}{\alpha_i} d(i, Q)\\
													 &= \sum_{i \in A} \frac{1}{\alpha_i} d(i, Q) + \sqrt{2}\sum_{i \in A} d(i, Q) - \sqrt{2}\sum_{i \in A} d(i, Q)\\
													 &= \sqrt{2}\sum_{i \in A} d(i, Q) - \sum_{i \in A} (\sqrt{2} - \frac{1}{\alpha_i})d(i, Q)\\
													 &\le \sqrt{2}\sum_{i \in A} d(i, Q) - \sum_{i \in A} (\sqrt{2} - \frac{1}{\alpha_i}) \frac{\alpha_i}{\alpha_i + 1} d(P, Q)\\
													 &= \sqrt{2}\sum_{i \in A} d(i, Q) - \sum_{i \in A} \frac{\sqrt{2}\alpha_i-1}{\alpha_i + 1} d(P, Q)
	\end{align*}

	We know that $\forall j \in B$ such that  $\beta_j \le \sqrt{2}$,

	$$ d(P, Q) \le d(j, P) + d(j, Q) = (\beta_j + 1) (j, Q) $$

	Bound the sum of $d(j, P)$ for all $j \in B$ that $\beta_j \le \sqrt{2}$,

	\begin{align*}
		\sum_{j \in B | \beta_j \le \sqrt{2}} d(j, P) &= \sum_{j \in B | \beta_j \le \sqrt{2}} \beta_j d(j, Q) \\
		&= \sum_{j \in B | \beta_j \le \sqrt{2}} \beta_j d(j, Q) + \sqrt{2} \sum_{j \in B | \beta_j \le \sqrt{2}} d(j, Q) - \sqrt{2} \sum_{j \in B | \beta_j \le \sqrt{2}} d(j, Q)\\
		&= \sqrt{2} \sum_{j \in B | \beta_j \le \sqrt{2}} d(j, Q) - \sum_{j \in B | \beta_j \le \sqrt{2}} (\sqrt{2} - \beta_j) d(j, Q) \\
		&\le \sqrt{2} \sum_{j \in B | \beta_j \le \sqrt{2}} d(j, Q) -\sum_{j \in B | \beta_j \le \sqrt{2}} \frac{\sqrt{2} - \beta_j}{\beta_j + 1} d(P, Q)
	\end{align*}

	$\forall j \in B$  such that $\beta_j > \sqrt{2}$,

  \begin{align*}
		d(j, P) &\le d(j, Q) + d(P,Q) \\
		(1 - \frac{1}{\beta_j}) d(j, P) &\le d(P, Q) \\
		d(j, P) &\le \frac{\beta_j}{\beta_j - 1} d(P, Q)
	\end{align*}

	We also know that $d(j, P) = \beta_j d(j, Q)$. Thus,
  \begin{align*}
		d(j, P) &= \frac{\sqrt{2}}{\beta_j} d(j, P) + (1 - \frac{\sqrt{2}}{\beta_j}) d(j, P)\\
		&= \frac{\sqrt{2}}{\beta_j} \times \beta_j d(j, Q) + \frac{\beta_j - \sqrt{2}}{\beta_j}d(j, P) \\
		&\le \sqrt{2} d(j, Q) + \frac{\beta_j - \sqrt{2}}{\beta_j} \times \frac{\beta_j}{\beta_j - 1} d(P, Q)\\
		&= \sqrt{2} d(j, Q) + \frac{\beta_j - \sqrt{2}}{\beta_j - 1} d(P, Q)
	\end{align*}

	Summing up for all $j \in B$  such that $\beta_j > \sqrt{2}$,
	$$ \sum_{j \in B | \beta_j > \sqrt{2}} d(j, P) \le \sqrt{2} \sum_{j \in B | \beta_j > \sqrt{2}} d(j, Q) + \sum_{j \in B | \beta_j > \sqrt{2}} \frac{\beta_j - \sqrt{2}}{\beta_j - 1} d(P, Q) $$

	Putting everything together,

	\begin{align*}
		&\sum_{i \in A} d(i, P) + \sum_{j \in B} d(j, P)\\
		&= \sum_{i \in A} \frac{1}{\alpha_i} d(i, Q) + \sum_{j \in B | \beta_j \le \sqrt{2}} d(j, P) + \sum_{j \in B | \beta_j > \sqrt{2}} d(j, P) \\
		&\le \sqrt{2} \sum_{i \in A \cup B} d(i, Q) + (- \sum_{i \in A} \frac{\sqrt{2}\alpha_i-1}{\alpha_i + 1} \text{\ } - \sum_{j \in B | \beta_j \le \sqrt{2}} \frac{\sqrt{2} - \beta_j}{\beta_j + 1} \text{\ } + \sum_{j \in B | \beta_j > \sqrt{2}} \frac{\beta_j - \sqrt{2}}{\beta_j - 1} ) d(P, Q)\\
		&\le \sqrt{2} \sum_{i \in A \cup B} d(i, Q)
	\end{align*}
\end{proof}

\begin{decision_rule}
\label{weighted_majority_rule_exact}
Given the exact preference strength of every voter for two candidates, assign weight $\frac{\sqrt{2} \alpha_i-1}{\alpha_i + 1}$ to each voter $i \in A$ such that $\alpha_i > \sqrt{2}$, and weight $\alpha_i - 1$ to each voter $i \in A$ such that $\alpha_i \le \sqrt{2}$. Assign weight $\frac{\sqrt{2} \beta_j-1}{\beta_j + 1}$ to each voter $j \in B$ such that $\beta_j > \sqrt{2}$ and weight $\beta_j - 1$ to each voter $j \in B$ such that $\beta_j \le \sqrt{2}$.\\
\end{decision_rule}

\begin{theorem}
	Using Weighted Majority Rule \ref{weighted_majority_rule_exact}, the distortion is at most $\sqrt{2}$ for two candidates, and this is the best bound possible.
\end{theorem}

\begin{proof}
	Without loss of generality, suppose $\sum_{i \in A| \alpha_i > \sqrt{2}} \frac{\sqrt{2} \alpha_i-1}{\alpha_i + 1} + \sum_{i \in A| \alpha_i \le \sqrt{2}} (\alpha_i - 1) \ge \sum_{j \in B | \beta_j > \sqrt{2}}  \frac{\sqrt{2} \beta_j-1}{\beta_j + 1} + \sum_{j \in B| \beta_j \le \sqrt{2}} (\beta_j - 1)$, and we choose $P$ as the winner.\\

	For $\alpha_i \le \sqrt{2}$, $(\alpha_i - 1) \le \frac{\sqrt{2} \alpha_i-1}{\alpha_i + 1}$. By the condition above,

	\begin{align*}
		\sum_{i \in A| \alpha_i > \sqrt{2}} \frac{\sqrt{2} \alpha_i-1}{\alpha_i + 1} + \sum_{i \in A| \alpha_i \le \sqrt{2}} (\alpha_i - 1) &\ge \sum_{j \in B | \beta_j > \sqrt{2}}  \frac{\sqrt{2} \beta_j-1}{\beta_j + 1} + \sum_{j \in B| \beta_j \le \sqrt{2}} (\beta_j - 1) \\
		\sum_{i \in A| \alpha_i > \sqrt{2}} \frac{\sqrt{2} \alpha_i-1}{\alpha_i + 1} + \sum_{i \in A| \alpha_i \le \sqrt{2}} \frac{\sqrt{2} \alpha_i-1}{\alpha_i + 1}  &\ge \sum_{j \in B | \beta_j > \sqrt{2}}  \frac{\sqrt{2} \beta_j-1}{\beta_j + 1} + \sum_{j \in B| \beta_j \le \sqrt{2}} (\beta_j - 1) \\
		\sum_{i \in A} \frac{\sqrt{2} \alpha_i-1}{\alpha_i + 1}  &\ge \sum_{j \in B | \beta_j > \sqrt{2}}  \frac{\beta_j - \sqrt{2}}{\beta_j - 1}  \\
		\sum_{i \in A} \frac{\sqrt{2} \alpha_i-1}{\alpha_i + 1}  &\ge \sum_{j \in B | \beta_j > \sqrt{2}}  \frac{\beta_j - \sqrt{2}}{\beta_j - 1} - \sum_{j \in B | \beta_j \le \sqrt{2}} \frac{(\sqrt{2} - \beta_j)}{\beta_j + 1}
	\end{align*}

	The second to last line follows because $\forall \beta_j \ge 1$, $ \frac{\sqrt{2}\beta_j-1}{\beta_j + 1} \ge \frac{\beta_j - \sqrt{2}}{\beta_j - 1}$. By Theorem \ref{thm-exact}, the distortion is at most $\sqrt{2}$.

	Now we show the claim above that $\forall \beta_j \ge 1$, $ \frac{\sqrt{2}\beta_j-1}{\beta_j + 1} \ge \frac{\beta_j - \sqrt{2}}{\beta_j - 1}$ to finish the proof.

	\begin{align*}
		(\beta_j - (\sqrt{2} + 1)) ^ 2 &\ge 0 \\
		\beta_j ^ 2 - 2(\sqrt{2} + 1) \beta_j + (\sqrt{2} + 1) ^ 2 &\ge 0 \\
		(\sqrt{2} - 1)\beta_j ^ 2 - 2(\sqrt{2} + 1)(\sqrt{2} - 1)\beta_j + (\sqrt{2} + 1) ^ 2 (\sqrt{2} - 1) &\ge 0 \\
		(\sqrt{2} - 1)\beta_j ^ 2 - 2 \beta_j + \sqrt{2} + 1 &\ge 0 \\
    \sqrt{2} \beta_j ^ 2 - \beta_j + 1 &\ge \beta_j ^ 2  + \beta_j - \sqrt{2}\\
		\sqrt{2} \beta_j ^ 2 - \beta_j - \sqrt{2}\beta_j + 1 &\ge \beta_j ^ 2  + \beta_j - \sqrt{2}\beta_j - \sqrt{2}\\
		(\sqrt{2}\beta_j-1)(\beta_j-1) &\ge (\beta_j- \sqrt{2})(\beta_j + 1)\\
		\frac{\sqrt{2}\beta_j-1}{\beta_j + 1} &\ge \frac{\beta_j - \sqrt{2}}{\beta_j - 1}
	\end{align*}

\end{proof}

\begin{corollary}
	Choosing a candidate from the uncovered set of a weighted majority graph obtained by using pairwise rule \ref{weighted_majority_rule_exact} results in distortion of at most 2 for any number of candidates.
\end{corollary}

This corollary is simply because if pairwise distortion is at most $\ds$, then the distortion of the uncovered set is at most $\ds^2$. While for other special cases we have better bounds on distortion with multiple candidates, for this case this general bound provides the best result.

%%%%%%%%%%%%%%%%%%%%%%%%%%%%%%%%%%%%%%%%%%%%%%%%%%%%
%\pagebreak

\section{Ideal Candidate Distortion}

In this section, we study the tradeoff between the winning candidate distortion $\actual$ and the ideal candidate distortion $\icd$. For instance $I=\{\voters, \cands, d\}$, suppose the winner is $P$, and the optimal candidate is $Q$. Denote the distortion of $P$ as $\actual = \frac{SC(P)}{SC(Q)}$. Recall the ideal candidate is the best possible point in the metric that minimizes the total social cost (this point may or may not be in \cands); we denote this point as $Z^*$ \footnote{In a Euclidean metric, Z* is the centroid}. Then the ideal candidate distortion $\icd$ of a candidate $P$ is $\icd = \frac{SC(P)}{SC(Z^*)}$.

We show that for {\em any instance}, we have that {\em either} the distortion of our mechanism is small, {\em or} the ideal candidate distortion of our winning candidate is small. In other words, we establish that the only time when the selected alternative is not similar to the absolutely best possible alternative (which may be even better than any of the candidates in the running), is when it is similar to the best candidate from the ones up for consideration. So, for cases when distortion is large, at least we have a ``consolation prize" that the chosen candidate is not too far from all possible alternatives, even the ones which the voters don't know about and do not express their preferences over.

To prove our results, we first need the following definition of a $\lambda$-bounded rule.

\begin{definition}
A majority rule is $\lambda$-bounded if \P beating \Q directly in pairwise comparison implies $SC(P) \leq SC(Q) + \lambda\cdot SC(Z)$ for any point $Z$ in the metric space.
\end{definition}

% Let $Z^*$ is the ideal candidate, let $Z$ be the best available candidate, and let $\ds = \frac{SC(P)}{SC(Z)}$.\footnote{In a Euclidean metric, Z* is the centroid}\\

\begin{theorem}
\label{thm-ideal-distortion}
If $P$ wins under a $\lambda$-bounded majority rule with two candidates, then $\icd = \frac{SC(P)}{SC(Z^*)} \leq \frac{\lambda \actual}{\actual - 1}$. If $P$ is in the uncovered set under a $\lambda$-bounded majority rule with multiple candidates, then $\icd = \frac{SC(P)}{SC(Z^*)} \leq \frac{2 \lambda \actual}{\actual - 1}$.
\end{theorem}

\begin{proof}
First consider the two candidates setting. We know that $SC(Q) = \frac{1}{\actual} SC(P)$, and by the definition of $\lambda$-bounded majority rule,

\begin{align*}
	SC(P) &\le SC(Q) + \lambda SC(Z^*) \\
	SC(P) &\le \frac{1}{\actual} SC(P) + \lambda SC(Z^*) \\
	SC(P) &\le \frac{\lambda \actual}{\actual - 1} SC(Z^*)
\end{align*}

Then for the multiple candidates setting, if $P$ is in the uncovered set, we know that either $P$ beats $Q$ directly (and we get the same bound as in the two candidates setting), or there exists a candidate $Y$, that $P$ beats $Y$ and $Y$ beats $Q$. Then by the definition of $\lambda$-bounded majority rule,

\begin{align*}
	SC(P) &\le SC(Y) + \lambda SC(Z^*) \\
	SC(P) &\le SC(Q) + \lambda SC(Z^*) + \lambda SC(Z^*) \\
	SC(P) &\le \frac{1}{\actual} SC(P) + 2\lambda SC(Z^*) \\
	SC(P) &\le \frac{2 \lambda \actual}{\actual - 1} SC(Z^*)
\end{align*}
\end{proof}

\begin{corollary}
	\label{corollary-ideal-distortion-1}
With only voters' ordinal preferences, in the two candidates setting, the majority winner has an ideal candidate distortion of $\icd \leq \frac{2 \actual}{\actual - 1}$. And in the multiple candidate setting, any candidate in the uncovered set has an ideal candidate distortion of  $\icd \leq \frac{4 \actual}{\actual - 1}$.
\end{corollary}
\begin{proof}
	By Lemma \ref{Goel}, the majority rule is $2$-bounded. Then we get the conclusion directly from Theorem \ref{thm-ideal-distortion}.
\end{proof}

Thus, in the usual ``ordinal preference" setting of \cite{Anshelevich2018} and \cite{goel2017metric}, either distortion of Copeland (or any candidate in the uncovered set) is actually bounded by $\actual \leq 3$ (instead of the worst-case of 5), or the ideal candidate distortion $\icd \leq 6$, which may not seem like a great bound, but is impressive because it means that the selected candidate is a factor of 6 away from all possibilities, ones that are not known to anyone, ones that no one expresses their preferences over, and ones that may arise sometime in the future. The only assumption required is that all the possible alternatives and voters lie in some arbitrary, possibly very high-dimensional, metric space.

The same tradeoff between $\actual$ and $\icd$ occurs if we have are given voters' preferences and a single threshold on preference strength, as in Section \ref{sec-1-tau}.

\begin{corollary}
	\label{corollary-ideal-distortion-1-tau}
With voter preferences and a single threshold $\thresh$, we use Weighted Majority Rule \ref{weighted_majority_rule_1_tau} to decide pairwise winners. Then in the two candidate setting, the winner has an ideal candidate distortion of $\icd \leq \frac{2 \actual}{\actual - 1}$ (Figure \ref{fig:two_candidates_tradeoff_1_tau}). And in the multiple candidate setting, any candidate in the uncovered set has an ideal candidate distortion of  $\icd \leq \frac{4 \actual}{\actual - 1}$ (Figure \ref{fig:multiple_candidates_tradeoff_1_tau}).
\end{corollary}

\begin{proof}
	By Theorem \ref{thm_PQZ_1_tau}, Weighted Majority Rule \ref{weighted_majority_rule_1_tau} is $2$-bounded. Then we get the conclusion directly from Theorem \ref{thm-ideal-distortion}.
\end{proof}

\begin{figure}[htb]
\centering
\begin{minipage}{.43\linewidth}
  \includegraphics[width=\linewidth]{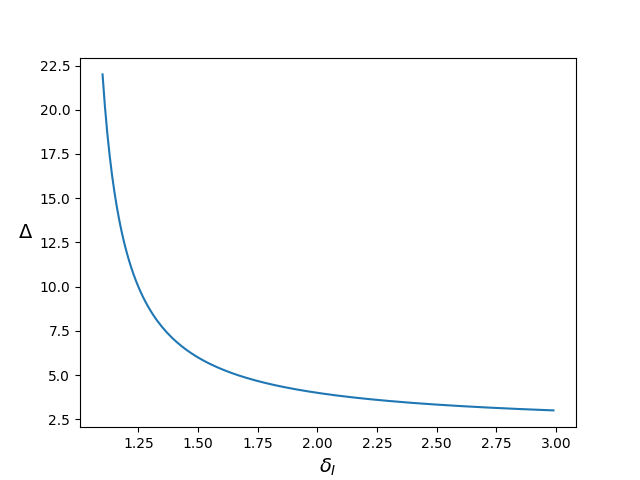}
  \captionof{figure}{Tradeoff between $\actual$ and $\icd$ with voter preferences and a single threshold $\thresh$ in the two candidates setting.}
  \label{fig:two_candidates_tradeoff_1_tau}
\end{minipage}
\hspace{.04\linewidth}
\begin{minipage}{.43\linewidth}
  \includegraphics[width=\linewidth]{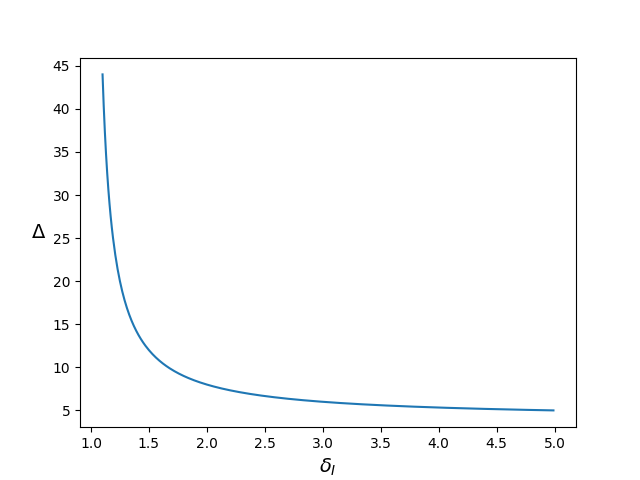}
  \captionof{figure}{Tradeoff between $\actual$ and $\icd$ with voter preferences and a single threshold $\thresh$ in the multiple candidates setting.}
  \label{fig:multiple_candidates_tradeoff_1_tau}
\end{minipage}
\end{figure}

\subsection{Ideal Candidate Distortion with Exact Preference Strengths}
In this section, we discuss the ideal candidate distortion when we know the voters' exact preference strength. We first show Weighted Majority Rule \ref{weighted_majority_rule_exact} is $(1+\sqrt{2})$-bounded, then get the ideal candidate distortion by Theorem \ref{thm-ideal-distortion}.

 Suppose there are two candidates $P$ and $Q$, and we are given the preference strength of every voter. Denote $A$ as the set of voters that prefer $P$ to $Q$, and $B$ as the set of voters that prefer $Q$ to $P$. The preference strength of any $i \in A$ is denoted as $\alpha_i$, and the preference strength of any $j \in B$ is denoted as $\beta_j$,

 We first present a lemma which allows us to charge voters in $B$ to voters in $A$; this lemma has not appeared previously and may be useful as a technique for proving other results as well.

\begin{lemma}
	\label{lemma-exact-ij}
	Given any voter $i \in A$, and voter $j \in B$, we have that $d(j, P) \le \frac{\beta_j (\alpha_i + 1)}{\alpha_i \beta_j - 1} d(i, j)$.
\end{lemma}

\begin{proof}
	By the triangle inequality, $d(i, j) \ge d(j, P) - d(i, P)$, and $d(i, j) \ge d(i, Q) - d(j, Q) = \alpha_i d(i, P) - \frac{1}{\beta_j}d(j, P)$. Thus,  $d(i, j) \ge \max \{ d(j, P) - d(i, P), \alpha_i d(i, P) -  \frac{1}{\beta_j} d(j, P) \}$.\\

	\textbf{Case 1. $ d(j, P) - d(i, P) \ge \alpha_i d(i, P) - \frac{1}{\beta_j} d(j, P)$}:\\

	\begin{align*}
		d(j, P) - d(i, P) &\ge \alpha_i d(i, P) - \frac{1}{\beta_j} d(j, P)\\
	  (1 + \frac{1}{\beta_j})d(j, P) &\ge (\alpha+1)d(i, P)\\
		 d(i,P) &\le \frac{\beta_j+1}{\beta_j(\alpha_i + 1)} d(j, P)
	\end{align*}

	\begin{align*}
		d(i, j) &\ge d(j, P) - d(i, P) \\
						&\ge d(j, P) - \frac{\beta_j+1}{\beta_j(\alpha_i + 1)} d(j, P)\\
						&= \frac{\alpha_i \beta_j - 1}{\beta_j(\alpha_i+1)} d(j, P)
	\end{align*}

 \textbf{Case 2. $ d(j, P) - d(i, P) < \alpha_i d(i, P) - \frac{1}{\beta_j} d(j, P)$}:\\

 \begin{align*}
	 d(j, P) - d(i, P) < \alpha_i d(i, P) - \frac{1}{\beta_j} d(j, P)\\
	 (1 + \frac{1}{\beta_j})d(j, P) < (\alpha+1)d(i, P)\\
		d(i,P) > \frac{\beta_j+1}{\beta_j(\alpha_i + 1)} d(j, P)
 \end{align*}

\begin{align*}
 d(i, j) &\ge \alpha_i d(i, P) - \frac{1}{\beta_j} d(j, P) \\
				 &>
				 %\frac{\beta_j+1}{\beta_j(\alpha_i + 1)} d(j, P) - d(j, P)\\&=
				 \frac{\alpha_i \beta_j - 1}{\beta_j(\alpha_i+1)} d(j, P)
\end{align*}

\end{proof}

\begin{lemma}
	\label{lemma-exact-2}
		If Weighted Majority Rule \ref{weighted_majority_rule_exact} selects P over Q, for any $i\in A$, $j \in B$, let $r_{ij} = \frac{(\beta_j-1) (\alpha_i + 1)}{\alpha_i \beta_j - 1}$. Then $\sum_{j \in B} \frac{1}{\sum_{i \in A} \frac{1}{r_{ij}}} \le 1 + \sqrt{2}$.
\end{lemma}

\begin{proof}
	\begin{align*}
		\sum_{j \in B} \frac{1}{\sum_{i \in A} \frac{1}{r_{ij}}}
		&= \sum_{j \in B} \frac{1}{\sum_{i \in A}  \frac{\alpha_i \beta_j - 1}{(\beta_j-1) (\alpha_i + 1)}} \\
		&= \sum_{j \in B} \frac{1}{\frac{\beta_j}{\beta_j-1} \sum_{i \in A}  \frac{\alpha_i - \frac{1}{\beta_j}}{\alpha_i + 1}} \\
		&\le \sum_{j \in B} \frac{1}{\frac{\beta_j}{\beta_j-1} \sum_{i \in A}  \frac{\alpha_i - 1}{\alpha_i + 1}} \\
		&= \frac{ \sum_{j \in B} \frac{\beta_j-1}{\beta_j} } { \sum_{i \in A} \frac{\alpha_i - 1}{\alpha_i + 1} } \\
	\end{align*}

	If $P$ is selected as the winner, it means

	\begin{equation*}
	\sum_{i \in A| \alpha_i > \sqrt{2}} \frac{\sqrt{2} \alpha_i-1}{\alpha_i + 1} + \sum_{i \in A| \alpha_i \le \sqrt{2}} (\alpha_i-1) \ge \sum_{j \in B | \beta_j > \sqrt{2}}  \frac{\sqrt{2} \beta_j-1}{\beta_j + 1} + \sum_{j \in B| \beta_j \le \sqrt{2}} (\beta_j-1)
\end{equation*}

	\begin{claim}
		$\forall \alpha_i \ge \sqrt{2}$, 	$\frac{\frac{\sqrt{2} \alpha_i-1}{\alpha_i + 1} } { \frac{\alpha_i - 1}{\alpha_i + 1} } = \frac{\sqrt{2} \alpha_i-1}{\alpha_i - 1} \le \sqrt{2}$
	\end{claim}

	\begin{claim}
		$\forall \alpha_i \leq \sqrt{2}$, $\frac{\alpha_i - 1}{ \frac{\alpha_i - 1}{\alpha_i + 1} } = \alpha_i + 1 \le 1 +
		 \sqrt{2}$
	\end{claim}

	\begin{claim}
		$\forall \beta_j \ge \sqrt{2}$, $\frac{\frac{\beta_j-1}{\beta_j}}{\frac{\sqrt{2} \beta_j-1}{\beta_j + 1} } \le 2 (\sqrt{2} - 1) < 1$
	\end{claim}

	\begin{proof}
		\begin{align*}
			((\sqrt{2} - 1)\beta_j - 1)^2 &\ge 0 \\
			(3 - 2\sqrt{2}) \beta_j ^ 2 - 2 (\sqrt{2} - 1) \beta_j + 1 &\ge 0 \\
			(2 - 2\sqrt{2}) \beta_j ^ 2 - 2 (\sqrt{2} - 1) \beta_j &\ge \beta_j ^ 2 - 1 \\
			2 (\sqrt{2} - 1) (\sqrt{2}\beta_j ^ 2 - \beta_j) &\ge \beta_j ^ 2 - 1 \\
			2 (\sqrt{2} - 1) \beta_j(\sqrt{2}\beta_j  - 1) &\ge (\beta_j + 1)(\beta_j - 1) \\
			2 (\sqrt{2} - 1) \frac{\sqrt{2} \beta_j-1}{\beta_j + 1} &\ge \frac{\beta_j-1}{\beta_j} \\
		\end{align*}
	\end{proof}

	\begin{claim}
		$\forall \beta_j \leq \sqrt{2}$, $\frac{\frac{\beta_j-1}{\beta_j}}{\beta_j-1} = \frac{1}{\beta_j} \le 1$\\
	\end{claim}

  By the four claims above,

  \begin{align*}
		\sum_{j \in B} \frac{\beta_j-1}{\beta_j} &\le \sum_{j \in B | \beta_j > \sqrt{2}}  \frac{\sqrt{2} \beta_j-1}{\beta_j + 1}  + \sum_{j \in B| \beta_j \le \sqrt{2}} (\beta_j-1)\\
		&\le \sum_{i \in A| \alpha_i > \sqrt{2}} \frac{\sqrt{2} \alpha_i-1}{\alpha_i + 1} + \sum_{i \in A| \alpha_i \le \sqrt{2}} (\alpha_i-1) \\
		&\le (1+\sqrt{2}) \sum_{i \in A} \frac{\alpha_i - 1}{\alpha_i + 1}\\
	\end{align*}

	Thus,

	\begin{align*}
		\sum_{j \in B} \frac{1}{\sum_{i \in A} \frac{1}{r_{ij}}}
		&\le \frac{ \sum_{j \in B} \frac{\beta_j-1}{\beta_j} } { \sum_{i \in A} \frac{\alpha_i - 1}{\alpha_i + 1} } \\
		&\le (1+\sqrt{2})
	\end{align*}
\end{proof}

\begin{theorem}
	\label{thm_PQZ_exact}
 If Weighted Majority Rule \ref{weighted_majority_rule_exact} selects P over Q, then $SC(P) \leq (1+\sqrt{2}) SC(Z) + SC(Q)$ where $Z$ can be any point in the metric space.
\end{theorem}

\begin{proof}

Let $A$ be the set of voters prefer $P$ to $Q$, and $B$ be the set of voters prefer $Q$ to $P$. If $P$ is selected as the winner, it means

	\begin{equation*}
	\sum_{i \in A| \alpha_i > \sqrt{2}} \frac{\sqrt{2} \alpha_i-1}{\alpha_i + 1} + \sum_{i \in A| \alpha_i \le \sqrt{2}} (\alpha_i-1) \ge \sum_{j \in B | \beta_j > \sqrt{2}}  \frac{\sqrt{2} \beta_j-1}{\beta_j + 1} + \sum_{j \in B| \beta_j \le \sqrt{2}} (\beta_j-1)
\end{equation*}

Select an arbitrary voter $i \in A$, and voter $j \in B$.

By Lemma \ref{lemma-exact-ij},

\begin{align*}
	d(j, P) &\le \frac{\beta_j (\alpha_i + 1)}{\alpha_i \beta_j - 1} d(i, j)\\
 \frac{\beta_j - 1}{\beta_j} d(j, P) &\le \frac{\beta_j - 1}{\beta_j} \times \frac{\beta_j (\alpha_i + 1)}{\alpha_i \beta_j - 1} d(i, j)\\
 \frac{\beta_j - 1}{\beta_j} d(j, P) &\le \frac{(\beta_j-1) (\alpha_i + 1)}{\alpha_i \beta_j - 1} d(i, j)\\
  \frac{\beta_j - 1}{\beta_j} d(j, P) &\le \frac{(\beta_j-1) (\alpha_i + 1)}{\alpha_i \beta_j - 1} (d(i, Z) + d(j, Z))\\
\end{align*}

Let $r_{ij} = \frac{(\beta_j-1) (\alpha_i + 1)}{\alpha_i \beta_j - 1}$,

$$\frac{1}{r_{ij}} \frac{\beta_j - 1}{\beta_j} d(j, P) \le (d(i, Z) + d(j, Z))$$

\begin{claim}
	For any $\alpha_i \ge 1$, $\beta_j \ge 1$, $r_{ij} \le 2$.
\end{claim}

\begin{proof}
	\begin{align*}
		\alpha_i &\ge 1\\
		\alpha_i(\beta_j+1) &\ge \beta_j + 1\\
		\alpha_i \beta_j + \alpha_i - \beta_j - 1 &\ge 0 \\
		2\alpha_i \beta_j - \alpha_i \beta_j + \alpha_i - \beta_j - 2 + 1  &\ge 0 \\
		2\alpha_i \beta_j - 2 &\ge \alpha_i \beta_j - \alpha_i + \beta_j - 1 \\
		2(\alpha_i \beta_j - 1) &\ge (\alpha_i + 1) (\beta_j - 1)\\
		 \frac{(\beta_j-1) (\alpha_i + 1)}{\alpha_i \beta_j - 1} & \le 2
	\end{align*}
\end{proof}

Sum up for all $i \in A$,

\begin{align*}
\sum_{i \in A} \frac{1}{r_{ij}} \frac{\beta_j - 1}{\beta_j} d(j, P) &\le \sum_{i \in A} (d(i, Z) + d(j, Z))\\
 \frac{\beta_j - 1}{\beta_j} d(j, P) &\le \frac{1}{\sum_{i \in A} \frac{1}{r_{ij}}} \sum_{i \in A} d(i, Z) + \frac{1}{\sum_{i \in A} \frac{1}{r_{ij}}} |A| d(j, Z)
\end{align*}

Then sum up for all $j \in B$,

\begin{align*}
 \sum_{j \in B} \frac{\beta_j - 1}{\beta_j} d(j, P) &\le \sum_{j \in B} \frac{1}{\sum_{i \in A} \frac{1}{r_{ij}}} \sum_{i \in A} d(i, Z) + |A| \sum_{j \in B} \frac{1}{\sum_{i \in A} \frac{1}{r_{ij}}} d(j, Z)\\
 &\le \sum_{j \in B} \frac{1}{\sum_{i \in A} \frac{1}{r_{ij}}} \sum_{i \in A} d(i, Z) + |A| \sum_{j \in B} \frac{1}{ \sum_{i \in A} \frac{1}{2}} d(j, Z)\\
 &\le \sum_{j \in B} \frac{1}{\sum_{i \in A} \frac{1}{r_{ij}}} \sum_{i \in A} d(i, Z) + 2 \sum_{j \in B} d(j, Z)\\
 &\le (1+\sqrt{2}) \sum_{i \in A} d(i, Z) + 2 \sum_{j \in B} d(j, Z)\\
\end{align*}

 Thus,
 \begin{align*}
\sum_{i \in A} d(i, P) +  \sum_{j \in B} d(j, P)
&= \sum_{i \in A} d(i, P) + \sum_{j \in B} (\frac{1}{\beta_j} + \frac{\beta_j - 1}{\beta_j}) d(j, P)\\
&\le \sum_{i \in A} d(i, Q) + \sum_{j \in B} d(i,Q) + \sum_{j \in B} \frac{\beta_j - 1}{\beta_j} d(j, P)\\
&\le \sum_{i \in A} d(i, Q) + \sum_{j \in B} d(i,Q) + (1+\sqrt{2}) \sum_{i \in A} d(i, Z) + 2 \sum_{j \in B} d(j, Z)
 \end{align*}

\end{proof}

\begin{corollary}
With every voter's exact preference strength, we use Weighted Majority Rule \ref{weighted_majority_rule_exact} to decide pairwise winners. In the two candidates setting, the winner has an ideal candidate distortion of $\icd \leq \frac{(\sqrt{2} + 1) \actual}{\actual - 1}$. And in the multiple candidate setting, any candidate in the uncovered set has an ideal candidate distortion of  $\icd \leq \frac{2(\sqrt{2} + 1) \actual}{\actual - 1}$.
\end{corollary}

\begin{proof}
	By Theorem \ref{thm_PQZ_exact}, Weighted Majority Rule \ref{weighted_majority_rule_exact} is $(\sqrt{2} + 1)$-bounded. Then we get the conclusion directly from Theorem \ref{thm-ideal-distortion}.
\end{proof}

\subsection{Ideal Candidate Distortion without knowing voter preferences}
\label{subset-ideal-tau}
In Section \ref{sec-1tau}, we discussed that with only one threshold $\thresh$, Weighted Majority Rule \ref{rule_majority_1thresh} is not $\lambda$-bounded for any constant $\lambda$. However, we can still get some tradeoff between the distortion $\actual$ of the winning candidate and the ideal candidate distortion $\icd$ for $\actual$ in a certain range. We will first show the relationship among $SC(P)$, $SC(Q)$, and $SC(Z)$ for any point $Z$ in the metric space by the following lemma.

\begin{lemma}
	\label{lemma-ideal-1thresh}
	Consider the setting with two candidates $P$, $Q$, and a single threshold $\thresh$. If $P$ pairwise beats $Q$ by Weighted Majority Rule \ref{rule_majority_1thresh}, then $SC(P) \le \thresh SC(Q) + 2 SC(Z)$ for any point $Z$ in the metric space.
\end{lemma}

\begin{proof}
Let $A$ denote the set of voters who prefer $P$ to $Q$ and have preference strength $\ge \thresh$, let $B$ denote the set of voters who prefer $Q$ to $P$ and have preference strength $\ge \thresh$, and let $C$ denote the rest of the voters. Because $P$ pairwise beats $Q$ by Weighted Majority Rule \ref{rule_majority_1thresh}, we know that $|A| \ge |B|$.

For voters in $A$ and $B$, by Lemma \ref{Goel}, $\sum_{i \in A+B} d(i, P) \le \sum_{i \in A+B} d(i, Q) + 2 \sum_{i \in A+B} d(i, Z)$. For voters in $C$, because we know that they do not strongly prefer $Q$ to $P$, it must be that $\sum_{i \in C} d(i, P) \le \thresh \sum_{i \in C} d(i, Q)$. Summing up all the voters in $A$, $B$, $C$, we get $SC(P) \le \thresh SC(Q) + 2 SC(Z)$.
\end{proof}

\begin{theorem}
	\label{thm-ideal-1thresh}
Consider the setting with two candidates $P$, $Q$, and a single threshold $\thresh$. Suppose $Z^*$ is the ideal possible candidate. If $P$ pairwise beats $Q$ by Weighted Majority Rule \ref{rule_majority_1thresh}, when the distortion $\actual \ge \thresh$, then $\icd = \frac{SC(P)}{SC(Z^*)} \leq \frac{2 \actual}{\actual - \thresh}$.
\end{theorem}

\begin{proof}
	By Lemma \ref{lemma-ideal-1thresh}, we know that:

	\begin{align*}
		SC(P) &\le \thresh SC(Q) + 2 SC(Z^*) \\
		SC(P) &\le \thresh \times \frac{1}{\actual} SC(P) + 2 SC(Z^*) \\
	(1 - \frac{\thresh}{\actual})	SC(P) &\le 2 SC(Z^*) \\
	SC(P) &\le \frac{2 \actual}{\actual - \thresh} SC(Z^*)
	\end{align*}

\end{proof}

Rewriting the bound of $\icd$ in terms of $\frac{\actual}{\thresh}$, $\icd \le \frac{2 \frac{\actual}{\thresh}}{\frac{\actual}{\thresh} - 1}$ when $\frac{\actual}{\thresh} \ge 1$. Thus the tradeoff is the same as in the case with only ordinal preferences being known (Figure \ref{fig:two_candidates_tradeoff_1_tau}), except replacing $\actual$ with $\frac{\actual}{\thresh}$. This makes sense since the case with only ordinal preferences is exactly the special case with a single threshold $\thresh=1$.

%Note that the tradeoff between $\icd$ and $\actual$ (instead of $\frac{\actual}{\thresh}$) is $\icd \leq \frac{2 \actual}{\actual - 1}$ (the same curve as in Figure \ref{fig:two_candidates_tradeoff_1_tau}) in the setting that we are only given voters' preferences.

\subsection{Ideal Candidate Distortion with General Thresholds}
In Section \ref{subset-ideal-tau}, we discussed that with only one threshold $\thresh$, there is a tradeoff between $\actual$ and $\icd$ when the distortion $\actual \ge \thresh$. Similarly, in the general setting when we are given $m$ thresholds $\{1 \leq \thresh_1 < \thresh_2 < \ldots < \thresh_m\}$, there is also a tradeoff between $\actual$ and $\icd$ when the distortion $\actual \ge \thresh_m$.

\begin{lemma}
	\label{lemma-ideal-general}
Consider the setting with two candidates $P$, $Q$, and thresholds $\{1 \leq \thresh_1 < \thresh_2 < \ldots < \thresh_m\}$. If $P$ pairwise beats $Q$ by Weighted Majority Rule \ref{rule_general}, then $SC(P) \le \thresh_m SC(Q) + 2 SC(Z)$ for any point $Z$ in the metric space.
\end{lemma}

\begin{proof}
  Let $A$ denote the set of all the voters that prefer $P$ to $Q$, and have preference strength $\ge \thresh_1$, i.e., $A = A_1 \cup A_2 \cup \dots \cup A_m$. Recall $A_l$ denotes the set of voters have preference strength $\alpha_i$ such that $\thresh_l \le \alpha_i < \thresh_{l+1}$. Similarly, define $B = B_1 \cup B_2 \cup \dots \cup B_m$.

	First we prove the size of $A$ is at least the size of $B_m$. In the proof of Theorem \ref{thm-distortion-general-rule}, we have discussed that Weighted Majority Rule \ref{rule_general} assigns heavier weights to voters with stronger preference strengths. Thus, the voters in $B_m$ and $A_m$ are assigned the heaviest weight. Remember voters in $C$ are assigned weight 0, so the winner is decided by voters in $A$ and $B$. If $P$ wins over $Q$, it must be the case that $|B_m| \le |A|$, because otherwise the total weight of voters in $B_m$ must be higher than the total weight of voters in $A$, and $Q$ would be the winner instead.

	For voters in $A$ and $B_m$, by Lemma \ref{Goel}, $\sum_{i \in A+B_m} d(i, P) \le \sum_{i \in A+B_m} d(i, Q) + 2 \sum_{i \in A+B_m} d(i, Z)$. For any other voter $i$ in $C$ or $B_l$ ($l < m$), we know that $d(i, P) \le \thresh_m d(i, Q)$. Summing up for all voters, we get $SC(P) \le \thresh_m SC(Q) + 2 SC(Z)$.
\end{proof}

\begin{theorem}
Consider the setting with two candidates $P$, $Q$, and thresholds $\{1 \leq \thresh_1 < \thresh_2 < \ldots < \thresh_m\}$. Suppose $Z^*$ is the ideal possible candidate. If $P$ pairwise beats $Q$ by Weighted Majority Rule \ref{rule_general}, when the distortion $\actual \ge \thresh_m$, then $\icd = \frac{SC(P)}{SC(Z^*)} \leq \frac{2 \actual}{\actual - \thresh_m}$.
\end{theorem}

\begin{proof}
The proof is exactly the same as for Theorem \ref{thm-ideal-1thresh}, by using Lemma \ref{lemma-ideal-general} instead of Lemma \ref{lemma-ideal-1thresh}.
\end{proof}

\section{Conclusion}
As we have shown, even a tiny amount of preference strength information allows us to significantly improve the distortion of social choice mechanisms. We quantify tradeoffs between the amount of information known about preference strengths and the achievable distortion and provide advice about which type of information about preference strengths seems to be the most useful.

When voters provide a single bit of extra preference strength information beyond their ordinal preferences, the distortion drops from 3 down to 1.83 between two candidates and from 4.236 down to 3.35 for multiple candidates if we can choose our threshold. When the exact preference strengths of all voters are known the distortion falls precipitously down to $\sqrt{2}$ for two candidates and $2$ for multiple candidates. In general, with only one or two chosen thresholds, one would not choose a threshold of $\thresh_1 = 1$, since it conveys less information than a slightly larger threshold. Intuitively, having a small barrier to voting that requires some effort to overcome means that only the votes of those with some stake in the outcome are included, but setting such a barrier too high can mean that many people with some interest in the decision are excluded. If we have more thresholds at our disposal we can further minimize distortion, but there are diminishing returns to additional thresholds. Considering the large improvements to distortion given just a single extra threshold, further information may not be worth the effort to obtain.

Unfortunately, one of the drawbacks to distortion as a measure efficiency is that it is not robust in practice. For example, with multiple candidates, the addition or subtraction of a single candidate or voter can cause the actual approximation achieved by Copeland in an instance to swing between 1 (optimal) and 5 (worst-case distortion). Our notion of {\em ideal candidate distortion} partly addresses this issue by showing that the distortion of Copeland (and other $\lambda$-bounded rules) can only be high when the winning candidate is within a constant factor of the ideal conceivable candidate, even if they are not a candidate and nothing about them is known. However, when distortion is low the ideal candidate distortion is unbounded in general. Therefore we observe a general tradeoff between the quality of the available candidates and how poorly we might possibly choose from among the candidates.

\subsection*{Acknowledgements} This work was partially supported by NSF award CCF-1527497.

%\nocite*
\bibliographystyle{plain}
\bibliography{Implicit_Utilitarian,Other,Preference_Strengths}

%\section{Appendix}

\end{document}